\documentclass[letterpaper]{article} 
\usepackage{aaai2026}  
\usepackage{times}  
\usepackage{helvet}  
\usepackage{courier}  
\usepackage[hyphens]{url}  
\usepackage{graphicx} 
\usepackage{natbib}  
\usepackage{caption} 
\usepackage{algorithm}
\usepackage{algpseudocode}
\usepackage{amsmath}
\usepackage{booktabs}
\usepackage[short]{optidef}
\usepackage{amsfonts}       
\usepackage{amsthm}
\usepackage{multirow}
\usepackage{makecell}
\newtheorem{theorem}{Theorem}
\newtheorem{lemma}{Lemma}
\usepackage[table]{xcolor} 
\usepackage{colortbl}     
\usepackage{pdflscape}
\usepackage{newfloat}
\usepackage{listings}
\DeclareCaptionStyle{ruled}{labelfont=normalfont,labelsep=colon,strut=off} 
\floatstyle{ruled}
\newfloat{listing}{tb}{lst}{}
\floatname{listing}{Listing}
\title{From Sequential to Recursive: Enhancing Decision-Focused Learning with Bidirectional Feedback}
\author{
    Xinyu Wang,
    Jinxiao Du,
    Yiyang Peng,
    Wei Ma\thanks{Corresponding Author.}
}
\affiliations{
    The Hong Kong Polytechnic University, Hong Kong SAR, China

    \{xinyuu.wang, jinxiao.du, yiyang.peng\}@connect.polyu.hk,
    wei.w.ma@polyu.edu.hk
}

\usepackage{bibentry}

\begin{document}

\maketitle

\begin{abstract}
Decision-focused learning (DFL) has emerged as a powerful end-to-end alternative to conventional predict-then-optimize (PTO) pipelines by directly optimizing predictive models through downstream decision losses. Existing DFL frameworks are limited by their strictly sequential structure, referred to as sequential DFL (S-DFL). However, S-DFL fails to capture the bidirectional feedback between prediction and optimization in complex interaction scenarios. In view of this, we first time propose recursive decision-focused learning (R-DFL), a novel framework that introduces bidirectional feedback between downstream optimization and upstream prediction. We further extend two distinct differentiation methods: explicit unrolling via automatic differentiation and implicit differentiation based on fixed-point methods, to facilitate efficient gradient propagation in R-DFL. We rigorously prove that both methods achieve comparable gradient accuracy, with the implicit method offering superior computational efficiency. Extensive experiments on both synthetic and real-world datasets, including the newsvendor problem and the bipartite matching problem, demonstrate that R-DFL not only substantially enhances the final decision quality over sequential baselines but also exhibits robust adaptability across diverse scenarios in closed-loop decision-making problems.
\end{abstract}


\section{Introduction}
Real-world operational tasks, such as vehicle routing, power generation scheduling, and inventory management, frequently involve decision-making under uncertainties \citep{donti2017task, qi2023practical, sadana2025survey, elmachtoub2022smart}. Conventional \emph{Predict-then-Optimize (PTO)} framework tackles such tasks by first predicting uncertain parameters through machine learning (ML), then solving optimization based on predictions. While widely adopted, PTO often yields suboptimal decisions by minimizing intermediate prediction errors rather than final decision quality \citep{elmachtoub2022smart, mandi2024decision, wang2025bridging, bertsimas2020predictive}. \emph{Decision-Focused Learning (DFL)} represents a significant advancement by embedding optimization directly into the learning process, thereby minimizing end-to-end decision errors rather than prediction losses \citep{amos2017optnet, mandi2022decision, kotary2021end, wang2025bridging}.

However, traditional DFL maintains a sequential prediction-then-optimize structure, which we term \emph{Sequential DFL (S-DFL)} in this paper. The one-way assumption in S-DFL: predictions inform optimization, but optimization outcomes do not influence subsequent predictions, renders S-DFL inadequate for complex, interactive systems where decisions generate feedback that should recursively refine predictions. Consider the ride-hailing matching problem as a canonical example: when the platform proposes driver-rider matching pairs, user accept/reject decisions provide immediate feedback that should inform the matching decisions \citep{qin2020ride}. This scenario exemplifies a broader class of multi-stakeholder adaptive systems, including dynamic pricing scenarios \citep{levina2009dynamic, jia2014online} and task allocation systems \citep{zhao2020preference}, which share three common characteristics: (\textit{i}) Initial predictions guide decisions, (\textit{ii}) Decision outcomes generate observable feedback, and (\textit{iii}) The feedback should recursively refine predictions. Such a recursive interaction creates a closed-loop coupling that S-DFL fails to capture, potentially leading to unstable training and suboptimal outcomes.

\textbf{The first research question arises: How can we effectively model bidirectional prediction-optimization systems with mutually interdependent components?} 
To answer the question, we first time propose \emph{Recursive Decision-Focused Learning (R-DFL)}, an extension of S-DFL that explicitly models the bidirectional feedback between upstream prediction and downstream optimization. As illustrated in Figure~\ref{fig: SFL_DFL structure}, R-DFL retains the forward prediction-to-optimization pipeline of S-DFL while innovatively introducing a feedback loop from optimization back to prediction. In this R-DFL framework: (\textit{i}) The optimization module leverages early-stage predictions to generate decisions, and (\textit{ii}) The decision outcomes, in turn, refine the predictions through feedback. This closed-loop architecture explicitly captures the dynamic interdependency characteristics inherent in bidirectional systems, aiming to enhance prediction accuracy and enable more adaptive and robust optimization compared with S-DFL.

\begin{figure}
    \centering
    \includegraphics[width=0.95\linewidth]{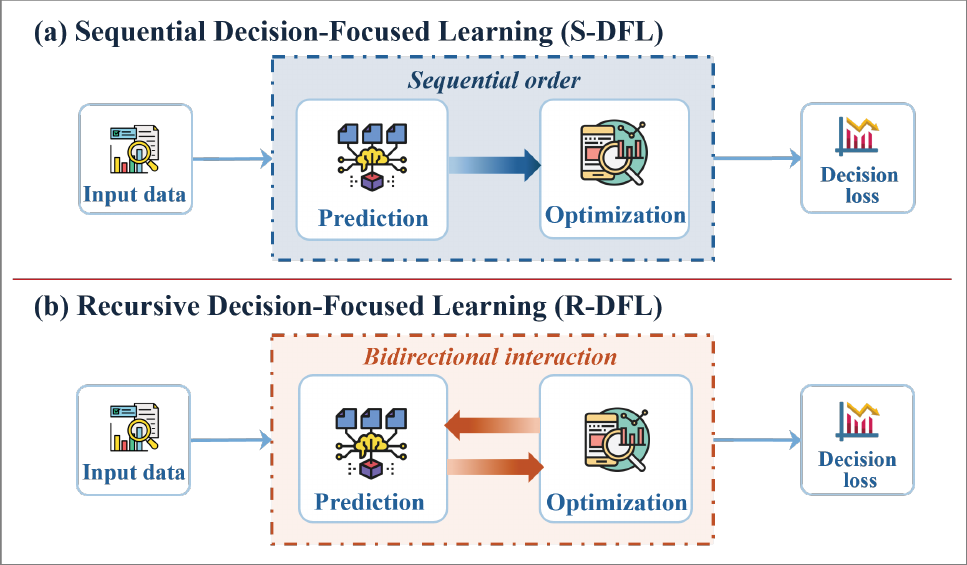}
    \caption{Comparison of S-DFL and R-DFL. R-DFL extends S-DFL by capturing the bidirectional interaction between prediction and optimization, where predictions are informed by both input features and optimization feedback.}
    \label{fig: SFL_DFL structure}
\end{figure}

\textbf{The second research question arises: How does gradient propagation operate in the bidirectional structure of the R-DFL framework?} The recursive framework introduces fundamental computational challenges: the cycle between prediction and optimization creates directed cyclic graphs, violating the directed acyclic graph (DAG) assumption underlying conventional deep learning (DL). Therefore, the resulting cyclic dependencies disrupt conventional gradient flow patterns and prevent a straightforward application of standard backpropagation via auto-differentiation (AD), necessitating specialized differentiation methods that preserve the recursive benefits while maintaining training stability and convergence.

To address the backpropagation challenge, we propose two differentiation schemes: explicit unrolling methods and implicit differentiation methods \citep{monga2021algorithm}. The explicit unrolling method handles cyclicity by sequentially unrolling the bidirectional prediction-optimization over a fixed number of iterations during the forward pass, allowing gradients to backpropagate through all unrolled layers via standard AD. While benefiting from straightforward implementation, the explicit unrolling method introduces significant computational overhead that scales with the unrolling depth, making it computationally complicated for deep recursive systems. 

In contrast, the implicit differentiation method formulates the coupled system as a fixed-point problem. During the forward pass, equilibrium states are computed using \texttt{RootFind} methods, while backward propagation leverages the Implicit Function Theorem \citep{krantz2002implicit} to derive gradients directly at equilibrium points. The implicit differentiation method offers substantial computational advantages by eliminating the need for unrolling computations, hence significantly reducing training time. 

We rigorously establish the equivalence of gradient updates between explicit unrolling and implicit differentiation methods through the Neumann series \citep{liao2018reviving, dimitrov2017computation}. This theoretical result guarantees consistent final accuracy performance regardless of the implementation choice of explicit unrolling and implicit differentiation methods in the R-DFL framework.

In summary, we present four major contributions:
\begin{itemize}
    \item We first time propose R-DFL, a novel recursive decision-focused learning framework that models bidirectional prediction-optimization feedback within the deep learning architecture, enabling closed-loop decision-making in dynamic systems.
    
    \item We develop two differentiation schemes for solving R-DFL with gradient propagation: (\textit{i}) A \emph{multi-layer} explicit unrolling method via auto-differentiation, and (\textit{ii}) A \emph{single-layer} implicit differentiation method that computes equilibrium gradients directly via the Implicit Function Theorem.
    
    \item We rigorously establish the \emph{gradient equivalence} between both differentiation schemes, guaranteeing consistent accuracy performance for explicit unrolling and implicit differentiation methods. 
    
    \item Extensive experiments on both synthetic and real-world datasets, including newsvendor and bipartite matching problems, demonstrate that the R-DFL framework significantly enhances the final decision quality. Comparing the two differentiating methods, while the explicit unrolling methods provide a more straightforward implementation, they often come with higher computational costs. The implicit differentiation methods require manual derivation of gradients but deliver superior computational efficiency.
\end{itemize}

\section{Related Works}
\subsection{Decision-Focused Learning}
DFL has emerged as an effective end-to-end framework to improve the decision quality by integrating predictive (prediction) and prescriptive (optimization) analytics within a unified DL architecture \citep{bertsimas2020predictive, mandi2022decision, postek2016multistage}. In this paper, we refer to this approach as Sequential DFL (S-DFL). S-DFL diverges from traditional PTO by embedding the optimization module as an implicit differentiable layer within the deep learning architecture, mapping input parameters directly to optimal solutions and incorporating decision quality into the loss function, allowing the predictive model to be trained using task-specific errors. Therefore, gradients can backpropagate efficiently through the optimization layer to parameters in the predictive models \citep{mandi2020smart, Berthet2020}. 
To improve final decision accuracy, one option is to develop novel surrogate loss functions \citep{kotary2022end, elmachtoub2022smart}.
Researchers have also developed specialized methods tailored to embedding distinct classes of optimization problems, including discrete optimization \citep{mandi2024decision, ferber2020mipaal}, linear programming \citep{mandi2020interior}, and quadratic programming \citep{amos2017optnet}. For handling non-smooth optimization problems, relaxation and approximation techniques have also been proposed \citep{Wilder2019, wang2019satnet}. 

The development of specialized software packages has significantly lowered the implementation barrier for S-DFL. Tools like \texttt{OptNet} for quadratic programming \citep{amos2017optnet}, \texttt{CvxpyLayers} for convex programming \citep{agrawal2019differentiable}, and \texttt{PyEPO} for linear programming \citep{tang2022pyepo} have enabled widespread real-world applications in domains such as inventory management and energy grid scheduling \citep{elmachtoub2022smart, donti2017task}. These advances have established S-DFL as a versatile framework for decision-making under uncertainty. \textbf{Despite its successes, S-DFL remains limited to sequential prediction-optimization pipelines and lacks generalization to bidirectional settings where prediction and optimization interact dynamically.}

\subsection{Explicit Unrolling Methods}
The derivation of gradients for solving DFL can be categorized into two distinct approaches based on their differentiation mechanisms: explicit unrolling and implicit differentiation. The explicit unrolling methods have found widespread applications in both time series modeling (e.g., recurrent neural networks like ResNet \citep{he2016deep}) and differentiable optimization. The explicit unrolling methods work by expanding either the temporal sequence (for time-series models) or the complete optimization procedure (for differentiable optimization) into an explicit computational graph during the forward pass. Each iteration is treated as a separate layer, creating a finite, unrolled representation that can be processed using standard automatic differentiation frameworks \citep{kotary2023backpropagation}. During backpropagation, gradients are calculated by sequentially applying the chain rule through all unrolled layers. Despite its widespread adoption, the explicit unrolling methods present several critical limitations. For problems requiring numerous iterations, the computational graph becomes excessively large, leading to (\textit{i}) substantial memory overhead, (\textit{ii}) increased computational time, and (\textit{iii}) potential numerical instability in gradient flow. These challenges mirror the vanishing and exploding gradient problems commonly observed in deep recurrent neural networks \citep{monga2021algorithm}, significantly constraining the applicability of explicit unrolling methods to complex, iterative problems.

\subsection{Implicit Differentiation Methods}
Implicit differentiation methods offer a powerful alternative to explicit approaches by replacing iterative procedures with a single implicit layer, significantly reducing the complexity of the corresponding computational graph. In this paradigm, the layer output is defined implicitly as the solution to an equilibrium equation rather than through an explicit computational graph. During backpropagation, gradients are computed directly at the solution point using the Implicit Function Theorem, which avoids the need to store intermediate states \citep{agrawal2019differentiable}. Implicit differentiation methods have found successful applications across multiple domains. In time-series modeling, the implicit methods have been employed in Neural ODEs \citep{chen2018neural} and Deep Equilibrium Models \citep{bai2019deep}, where fixed-point methods identify equilibrium states \citep{el2021implicit, zhang2019equilibrated, kazi2017implicitly, li2020end}. In differentiable optimization, implicit differentiation has been primarily implemented through differentiation of Karush-Kuhn-Tucker (KKT) conditions \citep{amos2017optnet, agrawal2019differentiable, Wilder2019}. Comparative studies highlight several advantages of implicit differentiation methods over explicit unrolling approaches. As noted by \citet{scellier2017equilibrium}, implicit differentiation methods provide better numerical stability while achieving superior memory and computational efficiency due to their compact representation. However, implicit methods typically require manual derivation of gradient expressions. Modern machine learning frameworks like PyTorch have addressed this challenge by providing native support for automatic differentiation and interfaces for custom gradient rules, making implicit differentiation methods more accessible to practitioners \citep{paszke2017automatic}. {\bf Overall, both explicit and implicit methods for R-DFL are rarely explored.}

\section{Preliminaries}
In this section, we briefly introduce notations and mathematical formulations of previous work on S-DFL.

\subsection{S-DFL}

The S-DFL pipeline comprises two essential modules: a predictive model $\mathcal{F}_\theta$ and a convex optimization model $\mathcal{G}$. 

\subsubsection{Predictive Model} Let $\mathcal{F}_\theta: \mathbb{R}^{d} \to \mathbb{R}^n$ be a parametric function with parameters $\theta\in \mathbb{R}^d$ that maps input features $\boldsymbol{v} \in \mathbb{R}^{d}$ to predicted outputs $\boldsymbol{\hat{c}}: = \mathcal{F}_{\theta}(\boldsymbol{v})$. $\hat{\boldsymbol{c}} \in \mathbb{R}^n$ then serve as parameters for the downstream optimization problem $\mathcal{G}$. 

\subsubsection{Convex Optimization Model} Define a convex program $\mathcal{G}: \mathbb{R}^n \to \mathbb{R}^n$ with objective $g$ and constraints $\boldsymbol{h}$ that takes input parameters $\hat{\boldsymbol{c}}$ and returns optimal solutions $\boldsymbol{x}^*({\hat{\boldsymbol{c}}})$ as:
\begin{align}
\label{eq: sdfl}
\mathcal{G}:  \boldsymbol{x}^*({\hat{\boldsymbol{c}}}) &= \underset{\boldsymbol{x} \in \mathcal{A}}{\text{argmin}} \ g(\boldsymbol{x}; \boldsymbol{\hat{c}}),
\end{align}
where the differentiable convex objective function is $g: \mathbb{R}^n \to \mathbb{R}$, $\boldsymbol{x} \in \mathcal{A}$ is the feasible region with linear constraints $\boldsymbol{h}(\boldsymbol{x}) \leq \boldsymbol{0}, \ \boldsymbol{h}: \mathbb{R}^n \to \mathbb{R}^m$. For simplicity, we omit the unpredictable parameters in the notation of $\mathcal{G}$. 

\subsubsection{Loss Function} The goal of S-DFL is to minimize the expected decision regret \citep{elmachtoub2022smart}: $\mathcal{R}(\boldsymbol{\hat{c}}, \boldsymbol{c}) =  g \bigl( \boldsymbol{x}^*(\boldsymbol{\hat{c}}), \boldsymbol{c}) - g(\boldsymbol{x}^*(\boldsymbol{c}), \boldsymbol{c}).$ If given a dataset $\mathcal{D} = \{ (v_n, c_n )\}_{n=1}^N$ with $N$ samples, the empirical loss function is: $\mathcal{L}_{(\boldsymbol{v},\boldsymbol{c})\in\mathcal{D}} 
    = \frac{1}{N} \sum_{n=1}^N  \mathcal{R}( \mathcal{F}_\theta(\boldsymbol{v}_n), \boldsymbol{c}_n).$

\subsubsection{Gradient Computation} To update the parameters $\theta$ in the predictive model via gradient descent methods, one can derive the gradient through the chain rule: $\frac{\partial \mathcal{L}}{\partial \theta} = \frac{\partial \mathcal{L}}{\partial \boldsymbol{x}^*} \frac{\partial \boldsymbol{x}^*}{\partial \boldsymbol{\hat{c}}} \frac{\partial \boldsymbol{\hat{c}}}{\partial \theta},$ where $\frac{\partial \mathcal{L}}{\partial \boldsymbol{x}^*}$ is the gradient of loss function w.r.t. the optimal decision, $\frac{\partial \boldsymbol{x}^*}{\partial \boldsymbol{c}}$ is obtained from the KKT condition by differentiating the convex optimization problem $\mathcal{G}$ or by unrolling the optimization procedures, and $\frac{\partial \boldsymbol{\hat{c}}}{\partial \theta}$ is the gradient of the output in the predictive model $\mathcal{F}_\theta$ w.r.t. parameters $\theta$.

\begin{figure*}[!ht]
\centering
\includegraphics[width=0.9\textwidth]{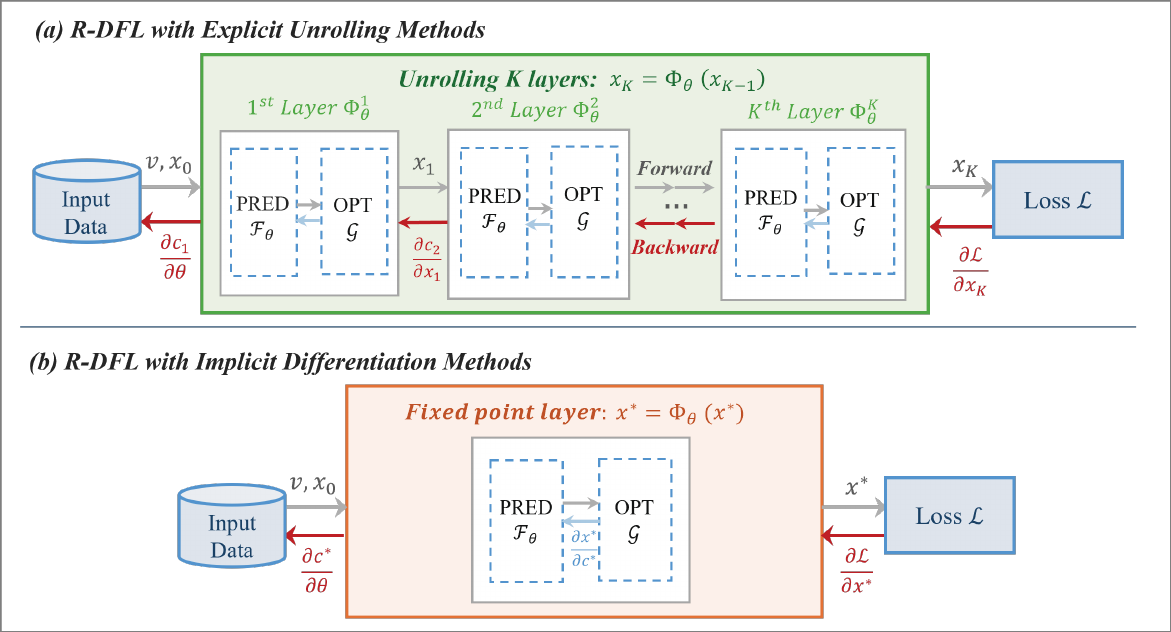}
\caption{Illustration of the R-DFL framework with explicit unrolling and implicit differentiation methods.}
\label{fig: R-DFL}
\end{figure*}

\section{Recursive Decision-Focused Learning}
In this section, we introduce the proposed Recursive Decision-Focused Learning (R-DFL) framework with two gradient propagation schemes: explicit unrolling and implicit differentiation, as illustrated in Figure \ref{fig: R-DFL}. Here we use the same notations of $\mathcal{F}_{\theta}$, $\mathcal{G}$, and $g$ in both S-DFL and R-DFL to smooth the presentation. Unlike conventional sequential approaches, the R-DFL models prediction and optimization with bidirectional feedback.The predictive model $\mathcal{F}_{\theta}$ takes inputs from both feature vector $\boldsymbol{v}$ and optimization results $\boldsymbol{x}$, then generates outputs $\hat{\boldsymbol{c}}$ to optimization model $\mathcal{G}$ as parameters.

\subsection{Overview of R-DFL Framework}
The R-DFL pipeline comprises two essential modules: a predictive model $\mathcal{F}_\theta$ and a convex optimization model $\mathcal{G}$. 

\subsubsection{Predictive Model} Let $\mathcal{F}_\theta: \mathbb{R}^{d+n} \to \mathbb{R}^{n}$ be a parametric predictive model with parameter $\theta \in \mathbb{R}^{d+n}$ that maps feature vectors $\boldsymbol{v} \in \mathbb{R}^{d}$ and optimization result $\boldsymbol{x} \in \mathbb{R}^n$ to predicted results $\boldsymbol{\hat{c}} \in \mathbb{R}^{n}$:\ $\hat{\boldsymbol{c}} = \mathcal{F}_{\theta}(\boldsymbol{x}, \boldsymbol{v})$.

\subsubsection{Convex Optimization Model} Define a convex optimization model $\mathcal{G}$ with differentiable objective function $g$ and linear constraints $\boldsymbol{h}$, takes the input from the predictive model $\hat{\boldsymbol{c}}$, and generates the optimal results:
\begin{align}
    \mathcal{G}:  \boldsymbol{x}^*({\hat{\boldsymbol{c}}}) &= \underset{\boldsymbol{x} \in \mathcal{A}}{\text{argmin}} \ g(\boldsymbol{x}; \boldsymbol{\hat{c}}),
\end{align}
where $g: \mathbb{R}^n \times \mathbb{R}^n \to \mathbb{R}$ is a differentiable convex function in $\boldsymbol{x}$ for fixed $\hat{\boldsymbol{c}}$, $\boldsymbol{x} \in \mathcal{A}$ is the feasible region with linear constraints $\boldsymbol{h}(\boldsymbol{x}) \leq \boldsymbol{0}, \ \boldsymbol{h}: \mathbb{R}^n \to \mathbb{R}^m$, $\boldsymbol{x}^*(\hat{\boldsymbol{c}})$ is the optimal solution. We consider the case with two assumptions in the optimization model $\mathcal{G}$: 

\begin{itemize}
    \item \textit{Single-valued, differentiable and convex}. The optimization model $\mathcal{G}$ induces a (potentially multi-valued) mapping from parameters to solutions. In our framework, we restrict $\mathcal{G}$ to be single-valued, differentiable, and convex \citep{diamond2016cvxpy}.
    \item \textit{Parameter only in objective}. The predicted parameter $\hat{\boldsymbol{c}}$ appears exclusively in the objective function, while constraints remain fixed, guaranteeing that the KKT conditions yield a well-defined differentiable mapping from predictions to decisions.
\end{itemize}

\subsubsection{Loss Function} To minimize the expected decision regret, the loss function of R-DFL is defined as:
\begin{align}
    \mathcal{L} = \mathbb{E}_{\boldsymbol{c}, \boldsymbol{v}} \left[ \mathcal{R}\bigl(\mathcal{F}_\theta(\boldsymbol{v}, \boldsymbol{x}^*(\boldsymbol{\hat{c}})), \boldsymbol{c}\ \bigr) \right].
\end{align} 

\subsubsection{Gradient Computation} To update the parameters $\theta$ in the predictive model, we need to compute the gradient $\frac{\partial \mathcal{L}}{\partial \theta}$ in the R-DFL framework. However, the recursive interaction between $\mathcal{F}_\theta$ and $\mathcal{G}$ in R-DFL prevents direct application of the chain rule like S-DFL. We address it through explicit unrolling and implicit differentiation methods below. 

\subsection{Explicit Unrolling Methods}

As illustrated in Figure \ref{fig: R-DFL}(a), the first explicit differentiation method unrolls $K$ iterations of the prediction-optimization cycle, treating each iteration as a separate layer in the computational graph. This allows gradient backpropagation through the entire unrolled sequence.

\subsubsection{Forward Pass} In each separate layer in the $K$ iterations, we compute predicition and optimization in sequence:
\begin{align}
\label{eq: cx}
    \hat{\boldsymbol{c}}_i = \mathcal{F}_\theta(\boldsymbol{x}_{i-1}, \boldsymbol{v}),\quad
    \boldsymbol{x}_i = \mathcal{G}(\hat{\boldsymbol{c}}_i),
  \quad i=1,2,\dots, K  
\end{align}

Define the composed prediction-optimization layer as $\Phi_\theta = \mathcal{G} \circ \mathcal{F}_\theta$, such that each layer in the unrolling process in Equation \ref{eq: cx} is compactly expressed to:
\begin{align}
\label{eq: x = g(f(x))}
\small
    \boldsymbol{x}_i = \mathcal{G}(\mathcal{F}_\theta(\boldsymbol{x}_{i-1},\boldsymbol{v})) = \Phi_\theta(\boldsymbol{x}_{i-1},\boldsymbol{v}), \quad i=1,2,\dots, K
\end{align}
where $\boldsymbol{x}_0$ is initialized randomly, and all $K$ layers in $\Phi_\theta$ share the same $\theta$.

Then the full unrolling computation of a total of $K$ layers in the forward pass yields (Omit $\boldsymbol{v}$ for simplicity) :
\begin{align}
    \boldsymbol{x}_K = \Phi_\theta(\boldsymbol{x}_{K-1}) \circ\ \Phi_\theta(\boldsymbol{x}_{K-2})\circ \cdots \ \circ \Phi_\theta(\boldsymbol{x}_0).
\end{align}

The final output $\boldsymbol{x}_K$ serves as input to the loss function $\mathcal{L}$.

\subsubsection{Backward Pass}
The gradient backpropagates sequentially through each composite layer $\Phi_\theta$ from layer $K$ to layer $1$ in the computational graph. The complete gradient $ \frac{\partial \mathcal{L}}{\partial \theta}$ for the explicit unrolling methods is formally given in Theorem \ref{thm: theorem_explicit}.

\begin{theorem}[\textbf{Gradient of Explicit Unrolling Methods}]
\label{thm: theorem_explicit}
Define the Jacobian of function $\Phi_\theta$  at $\boldsymbol{x}_{i}$ to be:
    \begin{align}
    \label{eq: J}
        J_{\Phi_\theta}|_{\boldsymbol{x}_{i}} &= 
        \frac{\partial \boldsymbol{x}_{i+1}}{\partial \boldsymbol{x}_{i}} =  
        \frac{\partial \Phi_\theta(\boldsymbol{x}_i)}{\partial \boldsymbol{x}_i}, \ \forall i \in \{1,...,K-1\},
    \end{align}

The Jacobian $J_{\Phi_\theta}|_{\boldsymbol{x}_{i}}$ is computed as:
    \begin{align}
    \label{eq: expand J_theta}
        J_{\Phi_\theta}|_{\boldsymbol{x}_{i}}  =  J_{\mathcal{G}}|_{\boldsymbol{c}_{i+1}} \cdot J_{\mathcal{F}_\theta}|_{\boldsymbol{x}_{i}}.   \ \forall i \in \{1,...,K-1\} 
    \end{align}
    
Then the loss gradient $\mathcal{L}$ w.r.t. $\theta$ at the final $K^{th}$ layer is:
    \begin{equation}
    \label{eq: main_theorem1_loss}
    \frac{\partial \mathcal{L}}{\partial \theta} = \frac{\partial \mathcal{L}}{\partial \boldsymbol{x}_K}  \Sigma_{i=1}^{K} \left(\left(\prod_{j=i+1}^{K} J_{\Phi_\theta}|_{\boldsymbol{x}_{j-1}}\right) 
        \frac{\partial \Phi_\theta(\boldsymbol{x}_{i-1})}{\partial\theta}\right).
    \end{equation}
    
\end{theorem}

\noindent \textbf{Proof.} See Appendix A.1. 

Theorem \ref{thm: theorem_explicit} presents the general form of the task-specific loss gradient in the backward pass of the explicit methods through the chain rule by AD. Appendix B.1 shows the pseudocode of the explicit unrolling methods. Below, we give brief explanations:
\begin{itemize}
    \item The product term $\prod_{j=i+1}^K J_{\Phi_\theta}|_{\boldsymbol{x}_{j-1}}$ captures the indirect cumulative gradients from path $\boldsymbol{x}_{K} \to \boldsymbol{x}_i$. 
    \item The partial derivative $\frac{\partial \Phi_\theta(\boldsymbol{x}_{i-1})}{\partial \theta}$ refers to the direct gradient dependence of $\boldsymbol{x}_i$ on $\theta$ at each unrolling step.
    \item Derivation of Equation \ref{eq: expand J_theta} is provided in Appendix A.5. The $J_{\mathcal{G}}|_{\boldsymbol{c}_{i}}$ is obtained from conducting sensitivity analysis on the KKT condition of optimization model $\mathcal{G}$, and $J_{\mathcal{F}_\theta}|_{\boldsymbol{x}_{i}}$ is the gradients in the predictive model. 
\end{itemize}

\subsection{Implicit Differentiation Methods}

The implicit differentiation method bypasses layer-by-layer unrolling by deriving directly at the equilibrium point through fixed-point methods, as illustrated in Figure \ref{fig: R-DFL}(b).

\subsubsection{Forward Pass}
The forward pass computes the equilibrium state $\boldsymbol{x}^*$ through \texttt{RootFind} of the composed model $\Phi_\theta$. At convergence, the equilibrium point satisfies:
\begin{align}
    &\boldsymbol{x}^* = \mathcal{G}(\mathcal{F}_\theta(\boldsymbol{x}^*, \boldsymbol{v})) = \Phi_\theta(\boldsymbol{x}^*, \boldsymbol{v}).
\end{align}

Let $\mathcal{H}_\theta$ denote the fixed-point layer (Omit $\boldsymbol{v}$ for simplicity):
\begin{align}
\label{eq: H_theta}
   \mathcal{H}_\theta = \boldsymbol{x}^* - \Phi_\theta(\boldsymbol{x}^*)  \rightarrow 0,
\end{align}
where the equilibrium point $\boldsymbol{x}^*$ is the root of $\mathcal{H}_\theta$. At convergence, prediction results $\boldsymbol{c}^*$ is obtained as: $\boldsymbol{c}^* = \mathcal{H}_\theta(\boldsymbol{x}^*)$. 

Note that alternative methods can be employed to achieve faster convergence guarantees, rather than relying on fixed-point iterations. For instance, if the composite function $\mathcal{H}$ is differentiable and convex, Newton's method or quasi-Newton methods can be used.

\subsubsection{Backward Pass} The gradient of the implicit methods directly back propagates at the equilibrium point. 
\begin{theorem}[\textbf{Gradient of Implicit Differentiation Methods}]
\label{thm: implicit}
The loss gradient back propagates at the equilibrium point $\boldsymbol{x}^*$ w.r.t. $\theta$ is given as:
\begin{align}
    \label{eq: main_Ltheta}
        \frac{\partial \mathcal{L}}{\partial \theta} = 
        \frac{\partial \mathcal{L}}{\partial \boldsymbol{x}^*}
        (I - J_{\Phi_\theta}|_{\boldsymbol{x^*}})^{-1}
        \frac{\partial \Phi_{\theta}(\boldsymbol{x}^*)}{\partial \theta},
    \end{align}
    where $J_{\Phi_\theta}|_{\boldsymbol{x^*}}$ is the Jacobian of $\Phi_\theta$ evaluated at $\boldsymbol{x}^*$.
\end{theorem}

\noindent\textbf{Proof.} See Appendix A.2. 

Theorem 2 presents the general form of the total loss gradient of the implicit method derived through the Implicit Function Theorem \citep{krantz2002implicit}. See Appendix B.2 for pseudocode of the implicit methods. Below, we give brief explanations:

\begin{itemize}
    \item In the forward \texttt{RootFind} procedures, we \textbf{disable} forward gradient tracking until finding the equilibrium point, ensuring the backward propagation operates directly from the equilibrium point $\boldsymbol{x}^*$ to the parameter $\theta$. 
    \item The implicit differentiation method computes exact gradients through the inverse Jacobian term $(I - J_{\Phi_\theta}|_{\boldsymbol{x^*}})^{-1}$, which implicitly captures the full unrolling structure of the forward pass. 
\end{itemize}

\subsection{Gradient Equivalence of Unrolling and Implicit Differentiation Methods}
We first present the convergence of explicit unrolling methods through Polyak-Łojasiewicz Inequality (PL) condition \citep{xiao2023alternating}, then prove the equivalent gradient of the explicit unrolling and implicit methods using the Neumann Series \citep{dimitrov2017computation}.

\begin{lemma}[\textbf{Convergence of Explicit Unrolling Methods}]
\label{lemma: pl}
Suppose the gradient of the differentiable function $\Phi_\theta$ is Lipschitz smooth, and $\Phi_\theta$ satisfies the PL condition with $\mu>0$:  
    \begin{align}
    \label{eq: pl}
        \frac{1}{2} \Vert \nabla \Phi_\theta(\boldsymbol{x}) \Vert^2 \geq \mu(\Phi_\theta(\boldsymbol{x}) - \Phi_\theta^*), 
    \end{align}
    where $\Phi_\theta^* = \underset{\boldsymbol{x}}{\inf} \ \Phi_\theta(\boldsymbol{x}) = \boldsymbol{x}^*$. 
    
Then $\Phi_\theta$ will converge to the fixed point $\boldsymbol{x}^*$ after infinite $K$ iterations $\lim_{K\to \infty} \boldsymbol{x}_K =  \boldsymbol{x}^*$. 
\end{lemma}
\noindent\textbf{Proof.} See Appendix A.4.

\begin{theorem}[\textbf{Gradient Equivalence of Explicit Unrolling and Implicit Differentiation Methods}]
\label{thm: equiv}

Assume that after infinite $K$ iterations, $\boldsymbol{x}_K$ converges to a fixed point \(\boldsymbol{x}^*\),:
\begin{align}
    &\lim_{K\to \infty} \boldsymbol{x}_K =  \boldsymbol{x}^*, &\boldsymbol{x}^* = \mathcal{G}\bigl(\mathcal{F}_\theta(\boldsymbol{x}^*)\bigr) = \Phi_\theta(\boldsymbol{x}^*). 
\end{align}

Then for any smooth loss $\mathcal{L}$, if the spectral radius holds $\rho(J_{\Phi_\theta}|\boldsymbol{x}^*) <1$, the loss gradients sasitfy:
\begin{align}
\small
  \frac{\partial \mathcal{L}}{\partial \theta}
  &= \lim_{K \to \infty}\frac{\partial \mathcal{L}}{\partial \boldsymbol{x}_K}  \Sigma_{i=1}^{K} \left(\left(\prod_{j=i+1}^{K} J_{\Phi_\theta}|\boldsymbol{x}_{j-1}\right) 
        \frac{\partial \Phi_\theta(\boldsymbol{x}_{i-1})}{\partial\theta}\right)\\
  &=\frac{\partial \mathcal{L}}{\partial \boldsymbol{x}^*}(I-J_{\Phi_\theta}|_{\boldsymbol{x}^*})^{-1}\frac{\partial \Phi_\theta(\boldsymbol{x}^*)}{\partial \theta}.
\end{align}

\end{theorem}

\noindent \textbf{Proof.} See Appendix A.5.

Theorem \ref{thm: equiv} states that the gradient obtained by unrolling infinite steps in Equation \ref{eq: main_theorem1_loss} agrees with the one given by implicit differentiation via fixed point in Equation \ref{eq: main_Ltheta}.

\begin{table*}[h]
    \centering
    \setlength{\tabcolsep}{4 pt} 
    \begin{tabular}{l|cc|cc|cc|cc|cc|cc}
        \toprule
        Dataset & \multicolumn{6}{c|}{\textbf{Newsvendor Problem}} & \multicolumn{6}{c}{\textbf{Bipartite Matching Problem}} \\
        \midrule
        Scalability & \multicolumn{2}{c|}{Small}& \multicolumn{2}{c|}{Mid} & \multicolumn{2}{c|}{Large} & \multicolumn{2}{c|}{Small} & \multicolumn{2}{c|}{Mid} & \multicolumn{2}{c}{Large} \\
        \midrule
        Decision variable& \multicolumn{2}{c|}{10} & \multicolumn{2}{c|}{50}& \multicolumn{2}{c|}{100} & \multicolumn{2}{c|}{16} & \multicolumn{2}{c|}{225}& \multicolumn{2}{c}{900} \\
        Constraints & \multicolumn{2}{c|}{32} & \multicolumn{2}{c|}{152} & \multicolumn{2}{c|}{302} & \multicolumn{2}{c|}{57} & \multicolumn{2}{c|}{706} & \multicolumn{2}{c}{2761} \\
        Jacobian matrix & \multicolumn{2}{c|}{$32\times 32$} & \multicolumn{2}{c|}{$152 \times 152$} & \multicolumn{2}{c|}{$302 \times 302$} & \multicolumn{2}{c|}{$57 \times 57$}& \multicolumn{2}{c|}{$706 \times 706$} & \multicolumn{2}{c}{2761 $\times$ 2761} \\
        \midrule
        Metrics & RMSE & Time & RMSE& Time & RMSE& Time & RMSE & Time & RMSE& Time & RMSE& Time \\
        \midrule
        PTO & 12.771  & - & 12.747 & -  & 12.684 & -   & 0.412 & -   & 0.232 & -   & 0.190 & -  \\
        S-DFL & 12.245 & -  & 12.536 & -  & 12.649 & -   & 0.408 & -   & 0.231 & -   & 0.187 & -   \\
       \underline{\textbf{R-DFL-U}} & 8.983 & 135 & 9.173  & 369 & 9.343 & 422  & \underline{\textbf{0.396}} & 65  & 0.222 & 432  & \underline{\textbf{0.170}} &  2704 \\
        \underline{\textbf{R-DFL-I}} & \underline{\textbf{8.831}}  & \underline{\textbf{118}} & \underline{\textbf{9.106}}  & \underline{\textbf{254}} & \underline{\textbf{9.327}}  & \underline{\textbf{369}} & 0.398  & \underline{\textbf{26}}  & \underline{\textbf{0.220}} & \underline{\textbf{65}} & 0.171  & \underline{\textbf{1867}}  \\
        \bottomrule
    \end{tabular}
    \caption{Performance of R-DFL framework with explicit unrolling and implicit differentiation methods on the newsvendor and bipartite matching problems. Jacobian size indicates dimension of matrix $J_{\mathcal{G}}|_{\boldsymbol{c}_{i+1}}$). Unit for time: seconds.}
\label{tab: large table}
\end{table*}

\section{Numerical Experiments}
We evaluate the effectiveness of the proposed R-DFL on two problems using various datasets at different scales: a modified classical multi-product newsvendor problem (MPNP) on a synthetic dataset and a bipartite matching problem (BMP) on a real-world dataset. All experiments are repeated 5 times with different random seeds, with average results reported.

\subsection{Benchmark Problems and Datasets}
\subsubsection{Multi-Product Newsvendor Problem with Synthetic Data}
We extend the classical MPBP \citep{donti2017task, cristian2023end, turken2012multi} to a recursive form (R-MPNP) with two stakeholders (retailers and suppliers). Suppliers determine order costs $\boldsymbol{c} \in \mathbb{R}^n$ responding to both retailer order quantities $\boldsymbol{x} \in \mathbb{R}^n$ and contextual features $\boldsymbol{v} \in \mathbb{R}^d$. Retailers determine optimal order quantities $\boldsymbol{x}^*$ while facing adaptive cost parameters set by suppliers. See Appendix C.1 for the mathematical formulations. Synthetic datasets across three scales ($n=10,50,100$) are generated, corresponding to Jacobian matrices $J_\mathcal{G}|_{\boldsymbol{c}_{i+1}}$ of size $32\times32$, $152\times152$, and $302\times302$ respectively in gradient computation. 

\subsubsection{Bipartite Matching Problem with Real-World Data}
Bipartite matching has broad applications in resource allocation, such as ride-hailing systems and housing assignments \citep{Wilder2019, benabbou2018diversity, zhao2019preference}. We examine the recursive BMP (R-BMP) in ride-hailing, which involves two parties: a centralized platform and individual participants (drivers and riders). The platform generates an allocation decision $\boldsymbol{x} \in \mathbb{R}^n$ to optimize driver-rider matching while estimating potential participant regret $\boldsymbol{c} \in \mathbb{R}^n$. Upon receiving their assignments, both drivers and riders experience realized regret based on the specific matches $\boldsymbol{x}$, creating a feedback loop. See Appendix C.2 for mathematical formulations. We evaluate R-DFL using real-world matching data from the NYC TLC trip dataset, under problem scales of ($4, 15, 30$) matches, corresponding to sizes of optimization problems $n = 16, 225, 900$ and Jacobian matrices $57\times57$, $706\times706$, and $2761\times2761$, respectively.

\subsection{Baselines}
We evaluate our R-DFL framework with two differentiation methods against the state-of-the-art S-DFL framework and PTO framework. All frameworks use identical predictive architectures and equivalent convex programming formulations. The baselines are:
\begin{itemize} 
    \item \textbf{R-DFL-U}: The recursive DFL with the explicit unroll methods through finite unrolling steps.
    \item \textbf{R-DFL-I}: The recursive DFL with the implicit differentiation methods backpropagating at the equilibrium point.
    \item \textbf{S-DFL}: The sequential DFL where the predictive model uses only exogenous features, without decision feedback.
    \item \textbf{PTO}: The conventional two-step approach with separate prediction and optimization, where the predictive model uses only exogenous features, without decision feedback.
\end{itemize}

\subsection{Evaluation Metrics}
We evaluate the baselines across two dimensions: (1) Decision accuracy, (2) Computational efficiency. While all baselines are evaluated on accuracy, the efficiency metrics focus specifically on the two differentiation methods of R-DFL.
\begin{itemize}
    \item \textbf{Accuracy}: RMSE loss for the final decision $\boldsymbol{x}$.
    \item \textbf{Efficiency}: Average training time per epoch (in seconds).
\end{itemize}

\subsection{Results}

The results in Table \ref{tab: large table} reveal three findings: (1) The proposed R-DFL framework with two differentiation methods consistently outperforms the S-DFL and PTO baselines across all datasets in accuracy, demonstrating that modeling of the recursive structure will significantly enhance the final decision quality in recursive decision-making problems with bidirectional feedback. (2) While maintaining comparable accuracy, the implicit differentiation methods achieve significantly less training time per epoch, 1.5 times faster than explicit unrolling methods in the large-scale problem of bipartite matching, highlighting their superior computational efficiency. (3) The predictive model in the newsvendor problem involves more predictive parameters than the bipartite matching problem, leading to longer training time despite smaller Jacobian matrices. Full table see Appendix D.3.

\subsection{Accuracy Comparison of the R-DFL}

Figure \ref{fig: qq plot} compares the decision result distributions of the R-DFL framework using two differentiation methods across training and test datasets in the newsvendor problem. The QQ plots reveal strong alignment between the two methods, particularly for small-scale datasets, indicating consistent decision outcomes. While minor deviations emerge in the largest dataset, the overall agreement suggests the robustness of both differentiation methods across varying problem scales, and either differentiation method can reliably support the R-DFL framework in practical applications.

\begin{figure}[ht]
    \centering
    \includegraphics[width=0.95\linewidth]{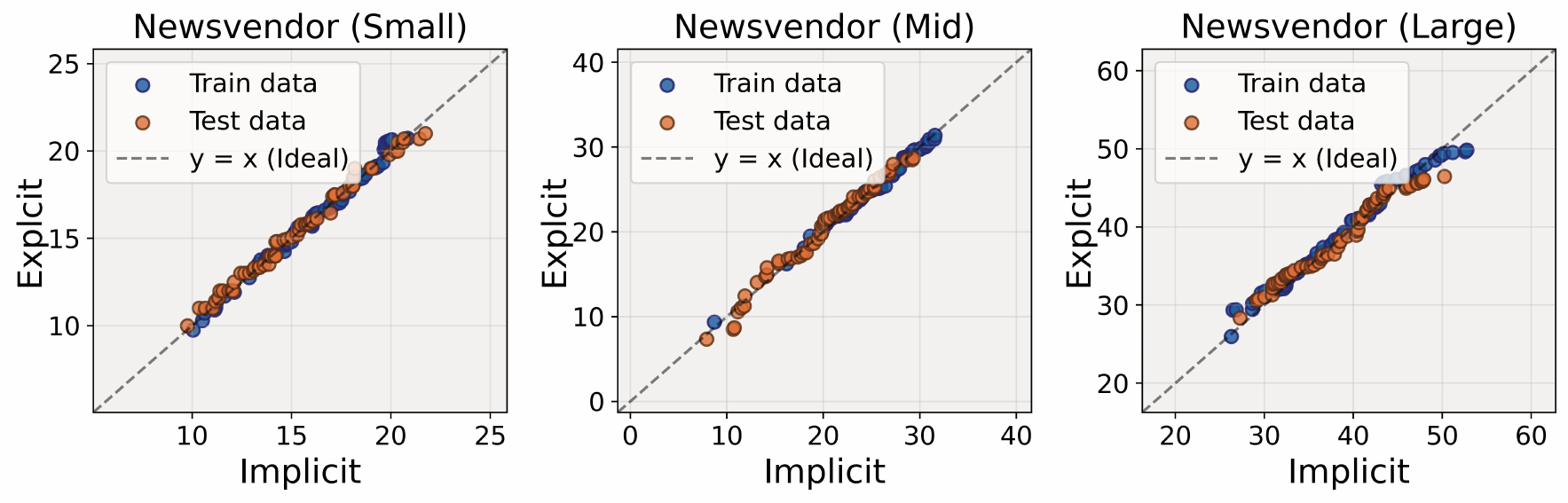}
    \caption{Accuracy comparison between R-DFL-U and R-DFL-I on newsvendor dataset across three scales.}
    \label{fig: qq plot}
\end{figure}

\subsection{Sensitivity Analysis on Unrolling Layers}
Figure \ref{fig: efficiency} presents a sensitivity analysis examining the impact of varying numbers of unrolling layers $\{5, 10, 15,20,25\}$ of both differentiation methods on the two datasets. Results demonstrate that: (1) Regarding accuracy, both the explicit and implicit methods achieve comparable RMSE values. (2) Regarding computational efficiency, the implicit methods exhibit superior performance compared to explicit unrolling with consistently less training time, particularly with the increment of unrolling layers. (3) Specifically, increasing the number of unrolling layers for explicit unrolling methods yields marginal accuracy improvements and results in substantial computational overhead, significantly prolonging training time.

\begin{figure}[h]
    \centering
    \includegraphics[width=1.0\linewidth]{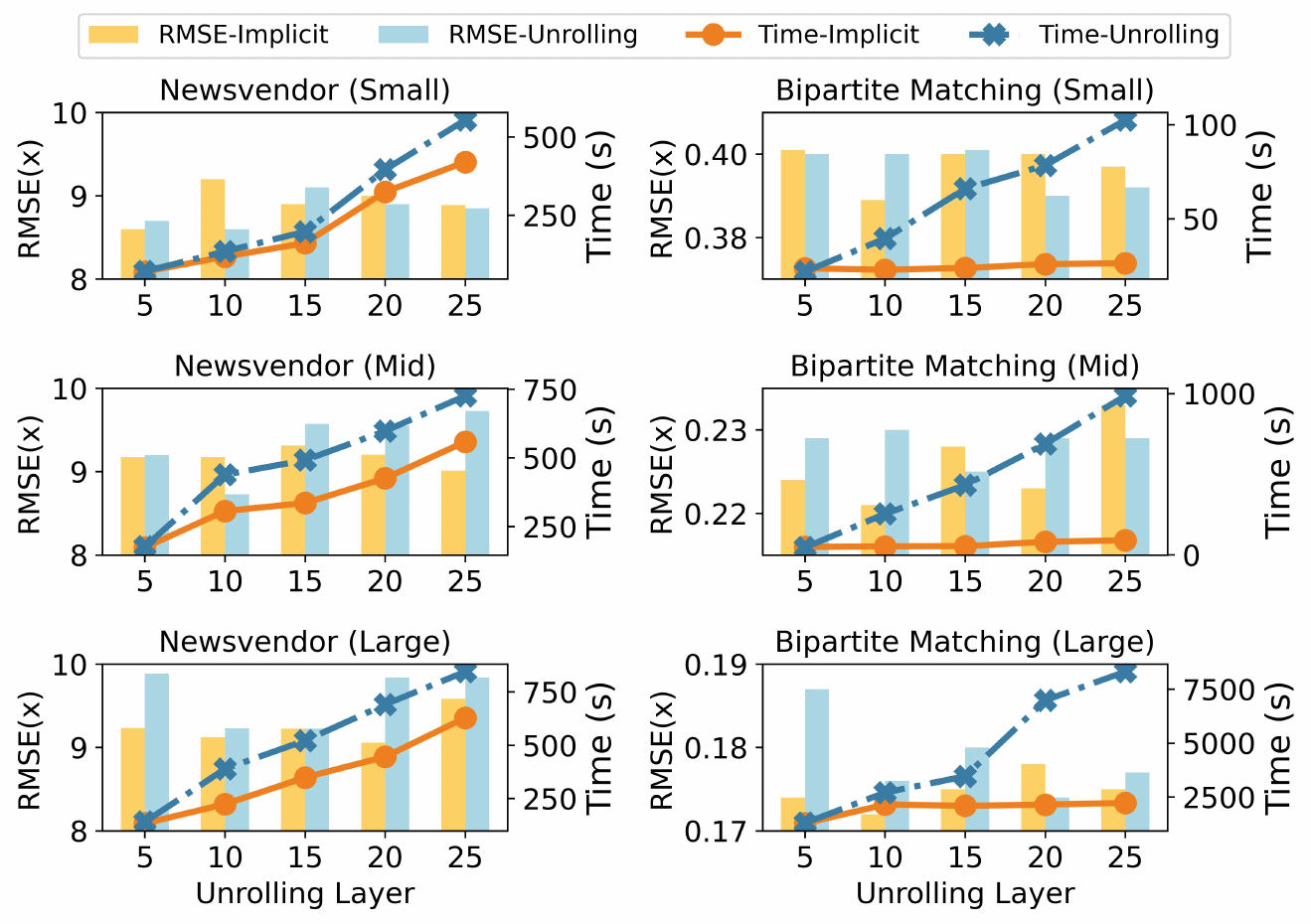}
    \caption{Sensitivity analysis of unrolling layers.}
    \label{fig: efficiency}
\end{figure}

\subsection{Robustness Check of the R-DFL}

Table \ref{tab: accuracy} reports results of three widely-used predictive models in two datasets, including \texttt{LSTM} \citep{hochreiter1997long}, \texttt{RNN} \citep{elman1990finding}, and \texttt{Transformer} \citep{vaswani2017attention} with 10 unrolling layers. The R-DFL framework with both differentiation methods shows higher accuracy than S-DFL and PTO frameworks in both experiments, representing the effectiveness and robustness of the proposed R-DFL framework. Full table see Appendix D.3.

\begin{table}[ht]
    \centering
    \small
    \setlength{\tabcolsep}{5 pt} 
    \begin{tabular}{l|cc|cc|cc}
        \toprule
        \multicolumn{7}{c}{\textbf{Newsvendor Problem (Large)}} \\
        \midrule
        Model & \multicolumn{2}{c|}{LSTM} & \multicolumn{2}{c|}{RNN} & \multicolumn{2}{c}{Transformer} \\
        \midrule
        Metrics & RMSE & Time & RMSE& Time & RMSE& Time \\
        \midrule
        PTO     & 12.034 & - & 12.842 & - & 14.071 & -\\
        S-DFL  & 12.693 & -& 12.583 & - & 14.040 & -\\
        \underline{\textbf{R-DFL-U}} & 10.112 & 531 & 10.872 & 510 & \underline{\textbf{11.231}} & 561 \\
        \underline{\textbf{R-DFL-I}} & \underline{\textbf{10.104}} & \underline{\textbf{355}}& \underline{\textbf{10.810}} & \underline{\textbf{340}} & 11.332 & \underline{\textbf{360}}\\
        \midrule
        \multicolumn{7}{c}{\textbf{Bipartite Matching Problem (Large)}} \\
        \midrule
        Model & \multicolumn{2}{c|}{LSTM} & \multicolumn{2}{c|}{RNN} & \multicolumn{2}{c}{Transformer} \\
        \midrule
        Metrics & RMSE & Time & RMSE& Time & RMSE& Time \\
        \midrule
        PTO     & 0.219 & - & 0.185 & -  & 0.193 & -\\
        S-DFL  & 0.230 & - & 0.187 & -  & 0.188 & - \\
        \underline{\textbf{R-DFL-U}} & 0.176 & 2704 & 0.176 & 2821  & 0.172 & 2821 \\
        \underline{\textbf{R-DFL-I}} & \underline{\textbf{0.174}} & \underline{\textbf{2093}} & \underline{\textbf{0.174}} & \underline{\textbf{2079}} & \underline{\textbf{0.166}} & \underline{\textbf{2078}}\\
        \bottomrule
    \end{tabular}
    \caption{Comparison with different predictive models.}
    \label{tab: accuracy}
\end{table}

\section{Conclusions}
While S-DFL has shown promise in improving decision quality by integrating optimization models in deep learning, its sequential structure fails to capture the critical interactions between prediction and optimization in closed-loop decision-making problems, hence limiting its applications. This paper introduces R-DFL, a novel framework that establishes bidirectional interaction between optimization and prediction, together with two mathematically-derived differentiation methods: explicit unrolling methods via auto-differentiation and implicit differentiation methods based on fixed-point analysis. Theoretical and numerical analyses demonstrate that considering the recursive structure substantially enhances the decision quality in problems with bidirectional feedback, regardless of prediction models. Furthermore, this paper reveals a trade-off between the two differentiation methods: while explicit unrolling methods offer a more straightforward implementation, implicit differentiation methods achieve superior computational efficiency. Beyond these specific applications, R-DFL represents a significant advancement in decision-focused learning by creating a truly unified prediction-prescription pipeline capable of modeling complex bidirectional systems, with promising extensions to broader classes of closed-loop decision-making problems under uncertainty. Future directions include extending R-DFL to handle stochastic recursive environments and developing a more versatile framework capable of addressing integer problems.

\section{Acknowledgments}
The work described in this paper is supported by grants from the Research Grants Council of the Hong Kong Special Administrative Region, China (Project No. PolyU/15206322 and PolyU/15227424), the Otto Poon Charitable Foundation Smart Cities Research Institute (SCRI) (U-CDC4), and the Research Centre for Digital Transformation of Tourism (RCDTT) (1-BBGU) at the Hong Kong Polytechnic University. The contents of this article reflect the views of the authors, who are responsible for the facts and accuracy of the information presented herein.
\bibliography{main}

\clearpage

\section{Appendix A. Gradient Derivation of R-DFL} 
This section presents the derivation of the gradients for explicit unrolling and implicit differentiation methods in the proposed R-DFL framework, the gradient equivalence between the two differentiation methods, and the derivation of the Jacobian Matrix.

\subsection{A.1 Proof of Theorem 1}
\setcounter{theorem}{0}

\begin{theorem}[\textbf{Gradient of Explicit Unrolling Methods}]
\label{thm: theorem_explicit}
Define the Jacobian of function $\Phi_\theta$ at $x_i$ to be:
    \begin{align}
    \label{eq: J}
        J_{\Phi_\theta}|_{\boldsymbol{x}_{i}} &= 
        \frac{\partial \boldsymbol{x}_{i+1}}{\partial \boldsymbol{x}_{i}} =  
        \frac{\partial \Phi_\theta(\boldsymbol{x}_i)}{\partial \boldsymbol{x}_i}, \ \forall i \in \{1,...,K-1\},
    \end{align}

And $J_{\Phi_\theta}|_{\boldsymbol{x}_{i}}$ will be obtained as follows:
    \begin{align}
    \label{eq: expand J_theta_app}
        J_{\Phi_\theta}|_{\boldsymbol{x}_{i}}  =  J_{\mathcal{G}}|_{\boldsymbol{c}_{i}} \cdot J_{\mathcal{F}_\theta}|_{\boldsymbol{x}_{i}}.   \ \forall i \in \{1,...,K-1\}
    \end{align}
    
Then the loss gradient $\mathcal{L}$ w.r.t. ($\theta$) at the final $K^{th}$ layer is given as:
    \begin{equation}
    \label{eq: theorem1_loss}
    \frac{\partial \mathcal{L}}{\partial \theta} = \frac{\partial \mathcal{L}}{\partial \boldsymbol{x}_K}  \Sigma_{i=1}^{K} \left(\left(\prod_{j=i+1}^{K} J_{\Phi_\theta}|_{\boldsymbol{x}_{j-1}}\right) 
        \frac{\partial \Phi_\theta(\boldsymbol{x}_{i-1})}{\partial\theta}\right).
    \end{equation}
\end{theorem}

\begin{proof}
    The loss gradient of $\mathcal{L}$ w.r.t. parameter $\theta$ is given as:
    \begin{align}
    \label{eq: dl/dtheta}
        \frac{\partial \mathcal{L}}{\partial \theta} = 
        \frac{\partial \mathcal{L}}{\partial \boldsymbol{x}_K}
        \frac{\partial \boldsymbol{x}_K}{\partial \theta},
    \end{align}

    In the explicit unrolling method, we have:
    \begin{align}
    \label{eq: x = phi(x)}
        \boldsymbol{x}_i = \Phi_\theta(\boldsymbol{x}_{i-1}), \forall  \ i = 1,..., K.
    \end{align}
    
    To compute $\frac{\partial\boldsymbol{x}_K}{\partial \theta}$, we differentiate Equation \ref{eq: x = phi(x)}
     w.r.t. $\theta$:
    \begin{align}
    \label{eq: dxi}
        \frac{\partial \boldsymbol{x}_i}{\partial \theta} = 
        \frac{\partial \Phi_\theta(\boldsymbol{x}_{i-1})}{\partial \boldsymbol{x}_{i-1}} 
        \frac{\partial \boldsymbol{x}_{i-1}}{\partial \theta} + 
        \frac{\partial \Phi_\theta(\boldsymbol{x}_{i-1})}{\partial\theta}.
    \end{align}

    Expand Equation \ref{eq: dxi} recursively for $i=0,1,...,K$, forming a chain structure as follows:
    \begin{align}
        \frac{\partial \boldsymbol{x}_0}{\partial \theta} =&0, \\
        \frac{\partial \boldsymbol{x}_1}{\partial \theta} =& 
        \frac{\partial \Phi_\theta(\boldsymbol{x}_{0})}{\partial\theta}, \\
        \frac{\partial \boldsymbol{x}_2}{\partial \theta} =& 
        \frac{\partial \Phi_\theta(\boldsymbol{x}_{1})}{\partial \boldsymbol{x}_{1}}
        \frac{\partial \Phi_\theta(\boldsymbol{x}_{0})}{\partial \theta} + 
        \frac{\partial \Phi_\theta(\boldsymbol{x}_{1})}{\partial\theta}, \\
        \frac{\partial \boldsymbol{x}_3}{\partial \theta} =& 
        \frac{\partial \Phi_\theta(\boldsymbol{x}_{2})}{\partial \boldsymbol{x}_{2}}
        \frac{\partial \Phi_\theta(\boldsymbol{x}_{1})}{\partial \boldsymbol{x}_{1}}
        \frac{\partial \Phi_\theta(\boldsymbol{x}_{0})}{\partial \theta} +  \\
        &\frac{\partial \Phi_\theta(\boldsymbol{x}_{2})}{\partial \boldsymbol{x}_{1}}
        \frac{\partial \Phi_\theta(\boldsymbol{x}_{1})}{\partial \theta}+
        \frac{\partial \Phi_\theta(\boldsymbol{x}_{2})}{\partial\theta}, \nonumber \\ 
        \cdots \nonumber
    \end{align}
    
    Through the iterations we have $\frac{\partial \boldsymbol{x}_K}{\partial \theta}$:
    \begin{align}
    \label{eq: dxk expand}
        \frac{\partial \boldsymbol{x}_K}{\partial \theta} 
        = & \Sigma_{i=1}^{K}\bigl( (\prod_{j=i+1}^{K} \frac{\partial \Phi_\theta(\boldsymbol{x}_{j-1})}{\partial \boldsymbol{x}_{j-1}} )
        \frac{\partial \Phi_\theta(\boldsymbol{x}_{i-1})}{\partial\theta} \bigr), 
    \end{align}

    Then substitute Equation \ref{eq: dxk expand} with Equation \ref{eq: J}:
    \begin{align}
    \label{eq: dxk/dtheta}
        \frac{\partial \boldsymbol{x}_K}{\partial \theta} 
        = \Sigma_{i=1}^{K}\bigl((\prod_{j=i+1}^{K} J_{\Phi_\theta}|_{\boldsymbol{x}_{j-1}}) 
        \frac{\partial \Phi_\theta(\boldsymbol{x}_{i-1})}{\partial\theta} \bigr),
    \end{align}
    where the product term $\prod_{j=i+1}^{K} J_{\Phi_\theta}|\boldsymbol{x}_{j-1}$ represents the indirect gradient contributions from the path $\boldsymbol{x}_{K} \to \boldsymbol{x}_i$, and the second term $\frac{\partial \Phi_\theta(\boldsymbol{x}_{i-1})}{\partial\theta}$ denotes the direct dependence of $\boldsymbol{x}_i$ on $\theta$ at each unrolling step $i$. 

    The loss gradient of $\mathcal{L}$ w.r.t. parameter $\theta$ is given as:
    \begin{align}
    \label{eq: dl/dtheta}
        \frac{\partial \mathcal{L}}{\partial \theta} = 
        \frac{\partial \mathcal{L}}{\partial \boldsymbol{x}_K}\Sigma_{i=1}^{K}\left(\left(\prod_{j=i+1}^{K} J_{\Phi_\theta}|_{\boldsymbol{x}_{j-1}}\right) 
        \frac{\partial \Phi_\theta(\boldsymbol{x}_{i-1})}{\partial\theta}\right).
    \end{align}
    
\end{proof}

\subsection{A.2 Proof of Theorem 2}
\begin{theorem}[\textbf{Gradient of Implicit Differentiation Methods}]
\label{thm: implicit}
The loss gradient back propagates at the equilibrium point $\boldsymbol{x}^*$ w.r.t. $\theta$ is given as:
\begin{align}
    \label{eq: Ltheta}
        \frac{\partial \mathcal{L}}{\partial \theta} = 
        \frac{\partial \mathcal{L}}{\partial \boldsymbol{x}^*}
        (I - J_{\Phi_\theta}|_{\boldsymbol{x^*}})^{-1}
        \frac{\partial \Phi_{\theta}(\boldsymbol{x}^*)}{\partial \theta},
    \end{align}
    where $J_{\Phi_\theta}|_{\boldsymbol{x^*}}$ is the Jacobian of $\Phi_\theta$ evaluated at $\boldsymbol{x}^*$.
\end{theorem}

\begin{proof}

After infinite iterations of $\boldsymbol{x}_{i+1} = \Phi_\theta(\boldsymbol{x}_i)$, the procedure will converge to some fixed output $\boldsymbol{x}^*$, with the property that:
\begin{align}
\label{eq: x* = g(f(x))}
    \boldsymbol{x}^* = \Phi_\theta(\boldsymbol{x}^*).
\end{align}

The loss gradient $\frac{\partial \mathcal{L}}{\partial \theta}$ at $\boldsymbol{x}^*$ will be:

\begin{align}
\label{eq: dldthete_x*}
    \frac{\partial \mathcal{L}}{\partial \theta} 
    &= \frac{\partial \mathcal{L}}{\partial \boldsymbol{x}^*}\frac{
    \partial \boldsymbol{x}^*}{\partial \theta}
\end{align}

To compute $\frac{\partial \boldsymbol{x}^*}{\partial \theta}$, differentiate both sides of Equation \ref{eq: x* = g(f(x))} w.r.t. $\theta$,

\begin{align}
    \frac{\partial \boldsymbol{x}^*}{\partial \theta}= 
    \frac{\partial \Phi_\theta(\boldsymbol{x}^*)}{\partial \theta} + 
    \frac{\partial \Phi_\theta(\boldsymbol{x}^*)}{\partial \boldsymbol{x}^*}
    \frac{\partial \boldsymbol{x}^*}{\partial\theta}.
\end{align}

Then we use the Implicit Function Theorem (IFT), so that we have:

\begin{align}
\label{eq: dx/dtheta}
    \frac{\partial \boldsymbol{x}^*}{\partial \theta}  = (I - \frac{\partial \Phi_\theta(\boldsymbol{x}^*)}{\partial \boldsymbol{x}^*})^{-1} \frac{\partial \Phi_\theta(\boldsymbol{x}^*)}{\partial \theta}.
\end{align}

Then substitute Equation \ref{eq: dx/dtheta} with Equation \ref{eq: J}:

\begin{align}
\label{eq: dx/dtheta_expanded}
\frac{\partial  \boldsymbol{x}^*}{\partial  \theta}  = ( I - J_{\Phi_\theta}|_{\boldsymbol{x}^*} )^{-1} \frac{\partial \Phi_\theta(\boldsymbol{x}^*)}{\partial \theta}.
\end{align}

Substitute Equation Equation \ref{eq: dldthete_x*} with \ref{eq: dx/dtheta_expanded}, the loss gradient $\frac{\partial \mathcal{L}}{\partial \theta}$ will be:

\begin{align}
\label{eq: dldtheta_x*}
    \frac{\partial \mathcal{L}}{\partial \theta}  = \frac{\partial \mathcal{L}}{\partial \boldsymbol{x}^*}(I-J_{\Phi_\theta}|_{\boldsymbol{x}^*})^{-1}
    \frac{\partial \Phi_\theta(\boldsymbol{x}^*)}{\partial \theta}.
\end{align}
\end{proof}

\subsection{A.3 Proof of Lemma 1}
\begin{lemma}[\textbf{Convergence of the Explicit Unrolling Methods}]
\label{lemma: pl}
Suppose the gradient of the differentiable function $\Phi_\theta$ is Lipschitz smooth, and $\Phi_\theta$ satisfies the PL condition with $\mu>0$:  
    \begin{align}
    \label{eq: pl}
        \frac{1}{2} \Vert \nabla \Phi_\theta(\boldsymbol{x}) \Vert^2 \geq \mu(\Phi_\theta(\boldsymbol{x}) - \Phi_\theta^*), 
    \end{align}
    where $\Phi_\theta^* = \underset{\boldsymbol{x}}{\inf} \ \Phi_\theta(\boldsymbol{x}) = \boldsymbol{x}^*$. 
    
Then $\Phi_\theta$ will converge to the fixed point $\boldsymbol{x}^*$ after infinite $K$ iterations $\lim_{K\to \infty} \boldsymbol{x}_K =  \boldsymbol{x}^*$.
\end{lemma}

\begin{proof}
    Let $f(\boldsymbol{x}) = \Vert \Phi_\theta(\boldsymbol{x}) - \boldsymbol{x} \Vert^2$, then $f(\boldsymbol{x}): \mathbb{R}^n \to \mathbb{R}$.
    If $\boldsymbol{x}^*$ is the fixed point, then we have:
    \begin{align}
        f^* = f(\boldsymbol{x}^*) = \underset{\boldsymbol{x}}{\inf} f = 0.
    \end{align}
    
    By Lipschitz smooth (L-smooth), we have:
    \begin{align}
        \Vert\nabla f(\boldsymbol{x}) - \nabla f(\boldsymbol{y})\Vert \leq L\Vert \boldsymbol{x} -\boldsymbol{y} \Vert, \quad \forall \boldsymbol{x}, \boldsymbol{y}.
    \end{align}
    
    Then consider the iteration with the step size $\eta \in (0, \frac{1}{L}]$:
    \begin{align}
        \boldsymbol{x}_{i+1} = \boldsymbol{x}_i - \eta \nabla \Phi_\theta(\boldsymbol{x}_i). 
    \end{align}
    By the L-smooth, we have:
    \begin{align}
    \label{eq: lemma 1 itera}
        f(\boldsymbol{x}_{i+1}) &\leq f(\boldsymbol{x}_{i}) - \eta \Vert \nabla f(\boldsymbol{x}_i) \Vert^2 + \frac{L\eta^2}{2}\Vert \nabla f(\boldsymbol{x}_i) \Vert^2  \nonumber\\
        &= f(\boldsymbol{x}_{i}) -\eta\left(1-\frac{L\eta}{2}\right) \Vert \nabla f(\boldsymbol{x}_i) \Vert^2.
    \end{align}
    
    As the function $f$ satisfy the Polyak-Łojasiewicz Inequility (PL) \citep{xiao2023alternating} conditions, there exists $\mu>0$:  
    \begin{align}
    \label{eq: pl}
        \frac{1}{2} \Vert \nabla f(\boldsymbol{x}) \Vert^2 \geq \mu\left( f(\boldsymbol{x}) - f^* \right), 
    \end{align}
    where $f^* = \underset{\boldsymbol{x}}{\inf} \ f(\boldsymbol{x^*}) = 0$. 

    Then substitute Equation \ref{eq: lemma 1 itera} with Equation \ref{eq: pl}:
    \begin{align}
    \label{eq: lemma 1 2 itera}
        f(\boldsymbol{x}_{i+1}) \leq f(\boldsymbol{x}_{i}) - 2\mu\eta\left( 1-\frac{L\eta}{2} \right) 
        (f(\boldsymbol{x}_{i}) - f^*).
    \end{align}

    Let $\alpha = 2\mu\eta(1-\frac{L\eta}{2})$, when $\eta \in (0,min(\frac{1}{L},\frac{1}{2\mu}))$, we have $0<\alpha<1$. Therefore:
    \begin{align}
    \label{eq: }
        f(\boldsymbol{x}_{i+1}) - f^* \leq (1- \alpha) (f(\boldsymbol{x_i}) - f^*).
    \end{align}
    Let $\eta = \frac{1}{L}$, we have:
    \begin{align}
    \label{eq: }
        f(\boldsymbol{x}_{i+1}) - f^* \leq (1- \frac{\mu}{L}) (f(\boldsymbol{x_i}) - f^*).
    \end{align}
    As $0\leq\frac{\mu}{L}\leq1$, we have:
    \begin{align}
         f(\boldsymbol{x}_{i+1}) - f^* \to 0.
    \end{align}
    As $f^* = 0$, we have $f(\boldsymbol{x}_{K})\to 0$.
    Therefore, recursively apply function $\Phi_\theta(\boldsymbol{x_i})$ for $K$ iterations if $K \to \infty$, we have:
    \begin{align}
        \lim_{K\to \infty} \Phi_\theta(\boldsymbol{x}_{K}) =  \Phi_\theta^* = \boldsymbol{x}^*.
    \end{align}
\end{proof}

\subsection{A.4 Proof of Theorem 3}
\begin{theorem}[\textbf{Gradient Equivalence of Explicit Unrolling and Implicit Differentiation Methods}]
\label{thm: equiv}

Assume that after infinite $K$ iterations, $\boldsymbol{x}_K$ converges to a fixed point $\boldsymbol{x}^*$:
\begin{align}
    &\lim_{K\to \infty} \boldsymbol{x}_K =  \boldsymbol{x}^*, \\
     &\boldsymbol{x}^* = \mathcal{G}\bigl(\mathcal{F}_\theta(\boldsymbol{x}^*)\bigr) = \Phi_\theta(\boldsymbol{x}^*). 
\end{align}

Then for any smooth loss $\mathcal{L}$, if the spectral radius holds $\rho(J_{\Phi_\theta}|\boldsymbol{x}^*) <1$, the loss gradients sasitfy:
\begin{align}
\small
  \frac{\partial \mathcal{L}}{\partial \theta}
  &= \lim_{K \to \infty}\frac{\partial \mathcal{L}}{\partial \boldsymbol{x}_K}  \Sigma_{i=1}^{K} \bigl((\prod_{j=i+1}^{K} J_{\Phi_\theta}|\boldsymbol{x}_{j-1}) 
        \frac{\partial \Phi_\theta(\boldsymbol{x}_{i-1})}{\partial\theta}\bigr)\\
  &=\frac{\partial \mathcal{L}}{\partial \boldsymbol{x}^*}(I-J_{\Phi_\theta}|_{\boldsymbol{x}^*})^{-1}\frac{\partial \Phi_\theta(\boldsymbol{x}^*)}{\partial \theta}.
\end{align}

\end{theorem}

\begin{proof}
    In the explicit unrolling method, if $K \to \infty$, we have:
    \begin{align}
    \lim_{K\to \infty} &\boldsymbol{x}_K \to \boldsymbol{x}^*,\\
    \lim_{K\to \infty}&J_{\Phi_\theta}|_{\boldsymbol{x}_K} \to \boldsymbol{J}_{\Phi_\theta}|_{\boldsymbol{x}^*} = J, \\
    \lim_{K\to \infty}&\frac{\partial \Phi_\theta(\boldsymbol{x}_K)}{\partial \theta}   \to \frac{\partial \Phi(\boldsymbol{x}^*)}{\partial \theta}.
    \end{align}
    
    Assume the spectral radius $\rho(J_{\Phi_\theta}|\boldsymbol{x}^*) = \rho(J) <1$ (ensuring stability of the fixed point), the product of Jacobians converges to a geometric series:
    \begin{align}
    \label{eq: geometry series}
        \prod_{j=i+1}^{K}J_{\Phi_\theta}|_{\boldsymbol{x}_{j-1}} \approx J^{K-i}.
    \end{align}

    Hence, in the unrolling step $K$, we have:
    \begin{align}
    \label{eq: dxk}
        \frac{\partial \boldsymbol{x}_K}{d\theta} 
        &= \Sigma_{i=1}^K \bigl( (\prod_{j=i+1}^{K} J_{\Phi_\theta}|_{\boldsymbol{x}_{j-1}})  \frac{\partial \Phi_\theta(\boldsymbol{x}_{i-1})}{\partial \theta} \bigr) \\
        &= \Sigma_{i=1}^K J^{K-i} \frac{\partial \Phi_\theta(\boldsymbol{x}^*)}{\partial \theta}.
    \end{align}    

    For $K \to \infty$, the Neumann series \citep{dimitrov2017computation} holds below:
    \begin{align}
    \label{eq: Neumann}
        \sum_{i=0}^\infty J^i = (I - J)^{-1}
    \end{align}
    
    Therefore, the loss gradient in the unrolling method converges to:
    \begin{align}
        \frac{\partial L}{\partial \theta}
        &=\lim_{K \to \infty}\frac{\partial \mathcal{L}}{\partial \boldsymbol{x}_K}  \Sigma_{i=1}^{K} \bigl((\prod_{j=i+1}^{K} J_{\Phi_\theta}|\boldsymbol{x}_{j-1}) 
        \frac{\partial \Phi_\theta(\boldsymbol{x}_{i-1})}{\partial\theta}\bigr)\\
        &= \frac{\partial L}{\partial \boldsymbol{x}^*}(I-J)^{-1} \frac{\partial \Phi_\theta(\boldsymbol{x}^*)}{\partial \theta} \\
        &= \frac{\partial L}{\partial \boldsymbol{x}^*}(I-J_\Phi|_{\boldsymbol{x}^*} )^{-1}\frac{\partial \Phi_\theta(\boldsymbol{x}^*)}{\partial \theta} .
    \end{align}

 This establishes the claimed gradient equivalence of explicit unrolling and implicit differentiation methods.
\end{proof}

\subsection{A.5 Derivation of the Jacobian Matrix}

For the Jacobian matrix $J_{\Phi_\theta}|_{\boldsymbol{x}_i}$ defined in Equation \ref{eq: J}, as we have $\Phi_\theta$ to be a composite function $\Phi_\theta = \mathcal{G}\circ \mathcal{F}_\theta$, of which $
    \boldsymbol{c} = \mathcal{F}_{\theta}(\boldsymbol{x}, \boldsymbol{v}), \boldsymbol{x} = \mathcal{G}(\boldsymbol{c})
$, we expand $J_{\Phi_\theta}|_{\boldsymbol{x}_i}$ as follows:
    \begin{align}
    \label{eq: expanded J}
    \small
        J_{\Phi_\theta}|_{\boldsymbol{x}_{i}} &=  
        \frac{\partial \Phi_\theta(\boldsymbol{x}_i)}{\partial \boldsymbol{x}_i} 
        =
        \frac{\partial \mathcal{G}(\boldsymbol{c}_{i+1})}{\partial \boldsymbol{c}_{i+1}}
        \frac{\partial \mathcal{F}_\theta(\boldsymbol{x}_{i})}{\partial \boldsymbol{x}_{i}},\\
        &=  J_{\mathcal{G}}|_{\boldsymbol{c}_{i+1}} \cdot J_{\mathcal{F}_\theta}|_{\boldsymbol{x}_{i}}, \   \forall i \in \{1,...,K-1\}.
    \end{align}

The $\mathcal{F}_\theta$ is the function of the predictive model (neural network), and hence $J_{\mathcal{F}_\theta}|_{\boldsymbol{x}_{i}}$ is very easy to obtain. The $\mathcal{G}$ is the optimization model and $J_{\mathcal{G}}|_{\boldsymbol{c}_{i}}$ is obtained through conducting a sensitivity analysis on Karush–Kuhn–Tucker (KKT) conditions of the optimization model $\mathcal{G}$. 

Below we show a general procedure of obtaining $\frac{\partial \mathcal{G}(\boldsymbol{c}_{i})}{\partial \boldsymbol{c}_{i}}$.
    
\textbf{STEP 1:} A convex optimization model $\mathcal{G}$ with differentiable convex objective $\boldsymbol{g}$ and convex constraints $\boldsymbol{h}$, $\boldsymbol{q}$, taking the parameter vector $\boldsymbol{c}$ only in the objective $g$, is defined as:
\begin{mini!}
{\boldsymbol{x}}
    {g(\boldsymbol{x};\boldsymbol{c})} 
    {\label{eq: general opt}}{\label{obj: general obj}}
    \addConstraint{\boldsymbol{h}(\boldsymbol{x})}{= \boldsymbol{0}} {\label{cst: eq}}
    \addConstraint{\boldsymbol{q}(\boldsymbol{x}) }{\leq \boldsymbol{0}} {\label{cst: ineq}}
\end{mini!}
where $\boldsymbol{x} \in \mathbb{R}^n$ is the decision variable, parameter $\boldsymbol{c} \in \mathbb{R}^n$, the objective function $g: \mathbb{R}^n\to \mathbb{R}$ is convex, $\boldsymbol{h}: \mathbb{R}^n \to \mathbb{R}^m$, constraints $\boldsymbol{q}: \mathbb{R}^n \to \mathbb{R}^p$ are affine. We consider the case when the solution mapping is single-valued, and parameters only appear in the objective function.

\textbf{STEP 2:} Give the Lagrangian multiplier function $\mathcal{K} $with dual variables $\boldsymbol{\lambda}, \boldsymbol{\mu}$ for constraints \ref{cst: eq} to \ref{cst: ineq}:

\begin{align}
    \mathcal{K}(\boldsymbol{x}, \boldsymbol{\lambda}, \boldsymbol{\mu};\boldsymbol{c}) = &g(\boldsymbol{x};\boldsymbol{c}) + \boldsymbol{\lambda}^T \boldsymbol{h}(\boldsymbol{x}) + \boldsymbol{\mu}^T\boldsymbol{q}(\boldsymbol{x}),
\end{align}
where $\boldsymbol{\lambda} \in \mathbb{R}^m, \boldsymbol{\mu} \in \mathbb{R}^p$.

\textbf{STEP 3: } Then, the KKT conditions are presented as follows:
\begin{align}
\small
    & \nabla_{\boldsymbol{x}} g(\boldsymbol{x};\boldsymbol{c}) + \sum_{i=1}^{m} \boldsymbol{\lambda}_{i} \nabla_{\boldsymbol{x}} \boldsymbol{h}_i(\boldsymbol{x}) +\sum_{j=1}^{p}\boldsymbol{\mu}_{j} \nabla_{\boldsymbol{x}} \boldsymbol{q}_j(\boldsymbol{x}) = 0, \label{eq: general_stationary} \\
    & \boldsymbol{h}_i(\boldsymbol{x}) = 0, \  i = 1,...,m \\
    & \boldsymbol{\mu}_j\boldsymbol{q}_j(\boldsymbol{x}) = 0, \ j = 1,...,p \\
    & \boldsymbol{\lambda} \geq 0 \\
    & \boldsymbol{\mu} \geq 0 \label{eq: general_kkt final}
\end{align}

\textbf{STEP 4:} Construct the KKT condition from Equation \ref{eq: general_stationary} to \ref{eq: general_kkt final} to an implicit function $\boldsymbol{\psi}(\boldsymbol{x}, \boldsymbol{\lambda}, \boldsymbol{\mu};\boldsymbol{c}) = 0$:
\begin{align}
\label{eq: Phi function}
\small
    &\boldsymbol{\psi}(\boldsymbol{x}, \boldsymbol{\lambda}, \boldsymbol{\mu};\boldsymbol{c}) = \nonumber \\
    \small
    &\left[ 
        \begin{array}{c}
        \small
             \nabla_{\boldsymbol{x}}g(\boldsymbol{x};\boldsymbol{c}) + \sum_{i=1}^{m} \boldsymbol{\lambda}_{i}\nabla_{\boldsymbol{x}}\boldsymbol{h}_i(\boldsymbol{x}) +\sum_{j=1}^{p}\boldsymbol{\mu}_{j}\nabla_{\boldsymbol{x}}\boldsymbol{q}_j(\boldsymbol{x}) = 0 \\
             \boldsymbol{h}(\boldsymbol{x}) = 0 \\
             \boldsymbol{q}(\boldsymbol{x}) \leq 0 \\
            \boldsymbol{\lambda} \geq 0  \\
            \boldsymbol{\mu} \geq 0  \\ 
             \boldsymbol{\mu}_j\boldsymbol{q}_j(\boldsymbol{x}) = 0 , \ j = 1,...,p
        \end{array} 
    \right]
\end{align}

For simplicity, we define:
\begin{align}
\small
    \phi(\boldsymbol{x};\boldsymbol{c}) =              \nabla_{\boldsymbol{x}}g(\boldsymbol{x};\boldsymbol{c}) + \sum_{i=1}^{m} \boldsymbol{\lambda}_{i}\nabla_{\boldsymbol{x}}\boldsymbol{h}_i(\boldsymbol{x}) +\sum_{j=1}^{p}\boldsymbol{\mu}_{j}\nabla_{\boldsymbol{x}}\boldsymbol{q}_j(\boldsymbol{x})
\end{align}

\textbf{STEP 5:} Differentiate the stationarity and the complementary slackness function in Equation \ref{eq: Phi function} w.r.t. $\boldsymbol{c}$:
\begin{align}
    \boldsymbol{M}
    \left[ 
        \begin{array}{c}
            d \boldsymbol{x} \\
            d \boldsymbol{\lambda} \\
            d \boldsymbol{\mu}
        \end{array}
    \right] = 
    \left[ 
        \begin{array}{c}
            -\nabla_{\boldsymbol{x},\boldsymbol{c}} g(\boldsymbol{x};\boldsymbol{c})  \\
            \boldsymbol{0} \\
            \boldsymbol{0}
        \end{array}
    \right] d \boldsymbol{c},     
\end{align}

where 
\begin{align}
\label{eq: general M}
\boldsymbol{M} = 
  \left[ 
  \small
  \setlength{\arraycolsep}{1.5pt}
        \begin{array}{ccc}
        \small
        \nabla_{\boldsymbol{x}} \phi(\boldsymbol{x};\boldsymbol{c}) &  \sum_{i=1}^{m}\nabla_{\boldsymbol{x}}\boldsymbol{h}_i(\boldsymbol{x}) & \sum_{j=1}^{p}\nabla_{\boldsymbol{x}}\boldsymbol{q}_j(\boldsymbol{x}) \\
        D( \boldsymbol{h}(\boldsymbol{x})) & \boldsymbol{0} & \boldsymbol{0}  \\
        D(\boldsymbol{\mu}) & \boldsymbol{0} & D(\boldsymbol{q}(\boldsymbol{x}))
        \end{array}    
    \right],
\end{align}
and $D(\cdot)$ denotes the diagonized matrix of $(\cdot)$.

\textbf{STEP 6: } Then we will obtain $\frac{\partial \boldsymbol{x}}{\partial \boldsymbol{c}}$ by computing the right side of Equation \ref{eq: general M}:

\begin{align}
\label{eq: general_MI}
    \renewcommand{\arraystretch}{1.5} 
    \left[ 
        \begin{array}{c}
            \frac{\partial \boldsymbol{x}}{\partial \boldsymbol{c}}  \\
            \frac{\partial \boldsymbol{\lambda}}{\partial \boldsymbol{c}} \\
            \frac{\partial \boldsymbol{\mu}}{\partial \boldsymbol{c}} 
        \end{array}
    \right] = -\boldsymbol{M}^{-1}
    \left[ 
    \begin{array}{c}
            \nabla_{\boldsymbol{x}\boldsymbol{c}} \ g(\boldsymbol{x};\boldsymbol{c}) \\
            \boldsymbol{0} \\
            \boldsymbol{0}
    \end{array}    
\right].
\end{align}

Then we will easily obtain $\frac{\partial \boldsymbol{x}_K}{\partial\boldsymbol{c}_K}$ via obtaining the first $n$ rows by Equation \ref{eq: general_MI}. 

\section{Appendix B. Pseudocodes}
This section presents the pseudocodes of explicit unrolling and implicit differentiation methods of the R-DFL framework.

\subsection{B.1 Pseudocode of Explicit Unrolling Methods}
In the explicit unrolling methods, we keep the gradient open all the time in the forward and backward pass, and the gradients will backpropagate through all unrolling layers to the parameter $\theta$.

\begin{algorithm}[ht]
\caption{R-DFL with Explicit Unrolling Method.}
\label{alg: explicit}
\textbf{Input}: Features $\boldsymbol{v}$, unrolling layer $K$, optimization model $\mathcal{G}$\\
\textbf{Parameter}: max epoch $N_\mu$, parameters $\theta$, step size $\alpha$ \\
\textbf{Output}: Optimal results $\boldsymbol{x}_K, \boldsymbol{c}_K$  
\begin{algorithmic}[1]
    \For{$\mu = 1$ \text{to} $N_\mu$} 
        \State Initialize $\boldsymbol{x}_0$ \Comment{\textit{Warm start initialization}}
        \For{$k = 1$ \text{to} $K$} \Comment{\textit{Forward: Unrolling}}
            \State $ \boldsymbol{c}_k \gets \mathcal{F}(\boldsymbol{v, \boldsymbol{x}}_{k-1}; \theta_k)$ \Comment{\textit{Open gradient}}
            \State $ \boldsymbol{x}_k \gets \mathcal{G}(\boldsymbol{c}_k)$
            \State \text{Update} $\boldsymbol{x}_{k-1} \gets \boldsymbol{x}_k$
        \EndFor
        \State $\theta_{k+1} \gets \theta_k - \alpha \frac{\partial \mathcal{L}}{\partial \theta_k}$ \text{by Eq. (\ref{eq: dl/dtheta})} \Comment{\textit{Backward by AD}}
    \EndFor
    \State \textbf{return} Optimal decision $\boldsymbol{x}_K, \boldsymbol{c}_K$
\end{algorithmic}
\end{algorithm}

\subsection{B.2 Pseudocode of Implicit Differentiation Methods}

In the implicit differentiation methods, the forward pass involves \textbf{closing} the gradient within the \texttt{RootFind} procedures until convergence or at the maximum number of iterating layers $K$. Subsequently, we manually open the gradient and construct the computational graph of prediction and optimization at the equilibrium point $\boldsymbol{x}^*$ and $\boldsymbol{c}^*$. During the backward pass, gradients propagate exclusively through the equilibrium point $\boldsymbol{x}^*$ and $\boldsymbol{c}^*$ and back to parameter $\theta$ through the differentiaiton function in Equation \ref{eq: dldtheta_x*}. The pseudocode of the implicit differentiation methods is presented below.

\begin{algorithm}[ht]
\caption{R-DFL with Implicit Differentiation Method.}
\label{alg: explicit}
    \textbf{Input}: Features $\boldsymbol{v}$, unrolling time $K$, optimization model $\mathcal{G}$\\
    \textbf{Parameter}: max epoch $N_\mu$, parameters $\theta$, step size $\alpha$\\
    \textbf{Output}: Optimal results $\boldsymbol{x}_K$ and $\boldsymbol{c}_K$   
    \begin{algorithmic}[1] 
    \For{$\mu \leq N_\mu$} 
        \State Initialize $\boldsymbol{x}_0$ 
        \For{$k \leq K$} \Comment{\textit{Fixed point RootFind}}
            \State $ \boldsymbol{c}_k \gets \mathcal{F}(\boldsymbol{v, \boldsymbol{x}}_{k-1};\theta)$ \Comment{\textit{Close gradient}}
            \State  $\boldsymbol{x}_k \gets \mathcal{G}(\boldsymbol{c}_k)$
            \State \text{Update} $\boldsymbol{x}_{k-1} \gets \boldsymbol{x}_k$
        \EndFor
        \State $ \boldsymbol{c}^* \gets F_\theta(\boldsymbol{v, \boldsymbol{x}}_{K})$ \Comment{\textit{Forward: Open gradient}}
        \State $\boldsymbol{x}^* \gets \mathcal{G}(\boldsymbol{c}^*)$
        \State Compute $\frac{\partial \mathcal{L}}{\partial \theta}$ \text{by Eq. (\ref{eq: dldtheta_x*})} \Comment{\textit{Backward manually}}
        \State $\theta^+ \gets \theta -\alpha\frac{\partial \mathcal{L}}{\partial \theta}$   
    \EndFor
    \State \textbf{return} Optimal decision $\boldsymbol{x}^*, \boldsymbol{c}^*$
\end{algorithmic}
\end{algorithm}

\section{Appendix C. Benchmark Problems}
\subsection{C.1 Newsvendor Problem}
\textbf{Problem Formulation}

In the modified recursive Multi-Product Newsvendor Problem (R-MPNP), we examine a supply chain system comprising two strategic stakeholders: a retailer and a supplier. The retailer must determine optimal daily order quantities for n distinct products, while the supplier sets the order prices $\boldsymbol{c}$ based on the retailer’s order quantities $\boldsymbol{x}$.

Specifically, we first model the relationship between order price $\boldsymbol{c}$ and order quantities $\boldsymbol{x}$ as follows: 
\begin{align}
\label{eq: nv_pred}
    \boldsymbol{c} = \mathcal{F}_\theta(\boldsymbol{\boldsymbol{x}, \boldsymbol{\alpha, \beta}}),
\end{align}
where $\boldsymbol{\alpha}, \boldsymbol{\beta}\in \mathbb{R}^n$ are parameters and $\boldsymbol{\alpha}, \boldsymbol{\beta} \geq 0$.

To determine the optimal order quantity $\boldsymbol{x}^*$ that minimizes the total order cost while satisfying operational constraints, we formulate the optimization problem $\mathcal{G}$:
\begin{mini!}
{\boldsymbol{x}}
    {\boldsymbol{c}^T \boldsymbol{x} } 
    {\label{eq: newsvendor}}{\label{obj: new_obj}}
    \addConstraint{\boldsymbol{1}^T \boldsymbol{x}}{\geq T_1} {\label{cst: news_min}}
    \addConstraint{\boldsymbol{1}^T \boldsymbol{x}}{\leq T_2} {\label{cst: news_max}}
    \addConstraint{\boldsymbol{x}}{\geq \boldsymbol{s}_1}{\label{cst: news_Nmin}}
    \addConstraint{\boldsymbol{x}}{\leq \boldsymbol{s}_2,}{\label{cst: news_Nmax}}
\end{mini!}
where $T_1, T_2 \in \mathbb{R}, \boldsymbol{s}_1, \boldsymbol{s}_2 \in \mathbb{R}^n$. The objective function is the total order cost of all products. Constraints \ref{cst: news_min} and \ref{cst: news_max} are the lower and upper bounds of the total order quantity; constraints \ref{cst: news_Nmin} and \ref{cst: news_Nmax} show the lower and upper order quantity of each product.

\textbf{Derivation of the Jacobian Matrix of R-MPNP}

To employ the implicit differentiation method, the critical step is obtaining the Jacobian Matrix $J_{\Phi_\theta}|_{\boldsymbol{x}^*}$. We follow the procedures in Appendix A.4 to compute.

To compute the required Jacobian at $\boldsymbol{x}^*$ of function $\Phi_\theta = \mathcal{G} \circ \mathcal{F}_\theta$, we need to compute:
\begin{align}
    J_{\Phi_\theta}|_{\boldsymbol{x}^*} &= J_{\mathcal{G}}|_{\boldsymbol{c}^*} \cdot J_{\mathcal{F}_\theta}|_{\boldsymbol{x}^*} \\
    &= 
    \frac{\partial \boldsymbol{x}^*}{\partial \boldsymbol{c}^*}
    \frac{\partial \boldsymbol{c}^*}{\partial \boldsymbol{x}^*},
\end{align}
where $\frac{\partial \boldsymbol{x}^*}{\partial \boldsymbol{c}^*}$ is obtained by conducting a sensitivity analysis on KKT conditions of the optimization model in Problem \ref{eq: newsvendor}, and $\frac{\partial \boldsymbol{c}^*}{\partial \boldsymbol{x}^*}$ is obtained easily by differentiating the predictive model in Equation \ref{eq: nv_pred}. 

Below we proceed the derivation to obtain $\frac{\partial \boldsymbol{x}^*}{\partial \boldsymbol{c}^*}$.
We first give the Lagrangian multiplier function $\mathcal{K}$ with dual variables $\lambda, \mu, \boldsymbol{\nu}, \boldsymbol{\eta}$ for constraints \ref{cst: news_min} to \ref{cst: news_Nmax}:
\begin{align}
\small
    \mathcal{K}(\boldsymbol{x}, \lambda, \boldsymbol{\nu}, \boldsymbol{\eta};\boldsymbol{c}) = &\boldsymbol{c}^T \boldsymbol{x} + \lambda(T_1 - \boldsymbol{1}^T\boldsymbol{x}) +\mu(\boldsymbol{1}^T \boldsymbol{x} -T_2) - \nonumber \\
    & \boldsymbol{\nu}^T (\boldsymbol{x} - \boldsymbol{s}_1) + \boldsymbol{\eta}^T(\boldsymbol{x}-\boldsymbol{s}_2),
\end{align}
where $\lambda, \mu \in \mathbb{R}$ and $\boldsymbol{\nu}, \boldsymbol{\eta} \in \mathbb{R}^n$. 

Then the KKT conditions are presented as follows:
\begin{align}
    &\nabla_{\boldsymbol{x}} \mathcal{K} = \boldsymbol{c}^T - \lambda\boldsymbol{1}^T +\mu\boldsymbol{1}^T- \boldsymbol{\nu}^T + \boldsymbol{\eta}^T = 0 \label{eq: news_stationarity} \\
    &\lambda \geq 0, \mu \geq 0, \boldsymbol{\nu}\geq 0, \boldsymbol{\eta} \geq0  \\
    &\lambda (T_1 - \boldsymbol{1}^T\boldsymbol{x}) = 0 \\
    &\mu(\boldsymbol{1}^T\boldsymbol{x}-T_2) = 0  \\   
    &\nu_i(\boldsymbol{s}_{1i} - x_i ) = 0,\  \forall \ i \in \{1,2,...,n\}  \\ 
    &\eta_i(x_i - \boldsymbol{s}_{2i}) = 0,\  \forall \ i \in \{1,2,...,n\}  \label{eq: news_kkt final}
\end{align}

Construct the KKT condition from Equation \ref{eq: news_stationarity} to \ref{eq: news_kkt final} to an implicit function $\boldsymbol{\psi}(\boldsymbol{x}, \lambda, \mu,\boldsymbol{\nu}, \boldsymbol{\eta}; \boldsymbol{c}) = 0$:
\begin{align}
\small
\label{eq: F function}
    &\boldsymbol{\psi}(\boldsymbol{x}, \lambda, \mu,\boldsymbol{\nu}, \boldsymbol{\eta}; \boldsymbol{c}) = \nonumber \\
    \small
    &\left[ 
        \begin{array}{c|c}
             \text{stationarity} & \boldsymbol{c}^T - \lambda\boldsymbol{1}^T +\mu\boldsymbol{1}^T- \boldsymbol{\nu}^T + \boldsymbol{\eta}^T = 0, \\
             \text{(primal} &(T_1 - \boldsymbol{1}^T\boldsymbol{x}) \leq 0, (\boldsymbol{1}^T\boldsymbol{x}-T_2) \leq 0, \\
            \text{feasibility)} &\boldsymbol{x} \geq \boldsymbol{s}_1, \boldsymbol{x}\leq \boldsymbol{s}_2, \\ 
            \text{dual feasibility} & \lambda, \mu, \boldsymbol{\nu}, \boldsymbol{\eta} \geq 0, \\ 
             \text{(complementary} &\lambda(T_1 - \boldsymbol{1}^T\boldsymbol{x}) = 0, \mu(\boldsymbol{1}^T\boldsymbol{x}-T_2) =0 \\
             \text{slackness)} &\nu_i(N_{1i} - x_i) = 0, \eta_i(x_i  - N_{2i}) = 0
        \end{array} 
    \right]
\end{align}

Differentiate the stationarity and the complementary slackness function w.r.t. $\boldsymbol{c}$:
\begin{align}
    \boldsymbol{M}
    \left[ 
        \begin{array}{c}
            d \boldsymbol{x} \\
            d\lambda \\
            d\mu\\
            d \boldsymbol{\nu}  \\
            d \boldsymbol{\eta}
        \end{array}
    \right] = 
    \left[ 
        \begin{array}{c}
            -\boldsymbol{I}  \\
            \boldsymbol{0} \\
            \boldsymbol{0} \\
            \boldsymbol{0} \\
            \boldsymbol{0} \\
        \end{array}
    \right] d\boldsymbol{c},     
\end{align}

\begin{align}
\label{M}
\boldsymbol{M} = 
  \left[ 
  \small
  \setlength{\arraycolsep}{1.5pt}
        \begin{array}{ccccc}
        \small
        \boldsymbol{0} & - \boldsymbol{1} & \boldsymbol{1} & -\boldsymbol{I} & \boldsymbol{I} \\
        -\lambda\boldsymbol{1}^T & T_1 - \boldsymbol{1}^T\boldsymbol{x} & 0 & \boldsymbol{0} & \boldsymbol{0}\\
        \mu \boldsymbol{1}^T & 0 & \boldsymbol{1}^T\boldsymbol{x} - T_2& \boldsymbol{0} & \boldsymbol{0} \\
        D(- \boldsymbol{\nu}) & \boldsymbol{0} & \boldsymbol{0} & D(\boldsymbol{s}_1-\boldsymbol{x}) & \boldsymbol{0} \\
        D( \boldsymbol{\eta}) & \boldsymbol{0} & \boldsymbol{0} & \boldsymbol{0} & D(\boldsymbol{x}-\boldsymbol{s}_2)
        \end{array}    
    \right],
\end{align}
where $D(\cdot)$ denotes the diagonal matrix of $(\cdot)$, and $\boldsymbol{M} \in \mathbb{R}^{(3N+2) \times (3N+2)}$.

Then we have: 
\begin{align}
\label{eq: nv_M^I}
    \left[ 
    \renewcommand{\arraystretch}{1.5} 
        \begin{array}{c}
            \frac{\partial\boldsymbol{x}}{\partial\boldsymbol{c}} \\
            \frac{\partial\boldsymbol{\lambda}}{\partial\boldsymbol{c}} \\
            \frac{\partial\boldsymbol{\mu}}{\partial\boldsymbol{c}}\\
            \frac{\partial\boldsymbol{\nu}}{\partial\boldsymbol{c}}  \\
            \frac{\partial\boldsymbol{\eta}}{\partial\boldsymbol{c}}  \\
        \end{array}
    \right] = -\boldsymbol{M}^{-1}\left[ 
        \begin{array}{c}
            \boldsymbol{I}  \\
            \boldsymbol{0} \\
            \boldsymbol{0} \\
            \boldsymbol{0} \\
            \boldsymbol{0} \\
        \end{array}
    \right].
\end{align}

Then $\frac{\partial \boldsymbol{x}^*}{\partial\boldsymbol{c}^*}$ is the first $n$ rows of Equation \ref{eq: nv_M^I}. 


\subsection{C.2 Bipartite Matching Problem}
\textbf{Problem Formulation}

In the modified recursive bipartite matching problem (R-BMP), we analyze a two-tier stakeholder system involving: (1) a centralized matching platform as the primary decision-maker, and (2) end-user participants (drivers and riders) as responding stakeholders. The platform stakeholder develops optimal matching allocations while endogenously modeling anticipated stakeholder regret. After matching, both drivers and riders may realize actual regret upon the matching outcomes. 

Mathematically, denote drivers as $I$ and riders as $J$. Let $\pi$ denote the regret function, $\boldsymbol{v} \in \mathbb{R}^{n \times d}$ denote the features, $\boldsymbol{x} \in \mathbb{R}^{n\times n}$ denote the match between drivers and riders. Then the regret function of drivers $i$ and riders $j$ is explicitly formulated as follows:
\begin{align}
    \pi_i :\ \ r_{i} &= f_{i}(x_{ij}, \boldsymbol{v}_i), \\
    \pi_j : \ r_{j} & = f_j(x_{ij}, \boldsymbol{v}_j), 
\end{align}
where $\boldsymbol{v}_i$ includes pickup time and urgency rate, $\boldsymbol{v}_j$ includes pickup time, peak and fatigue, $\boldsymbol{v} = \{\boldsymbol{v}_i, \boldsymbol{v}_j\}$. 

We then define a generalized cost $\boldsymbol{c}: c_{ij} = r_i + r_j$ to combine the regret of riders and drivers, and hence:
\begin{align}
    \boldsymbol{c}= f_\theta(\boldsymbol{x}, \boldsymbol{v}),
\end{align}
where $\boldsymbol{c} \in \mathbb{R}^{n\times n}$ refers to the generalized matching cost on the match between set ${i,j}$, and $\theta$ includes all the parameters in functions $f_i, f_j$.

Therefore, in the predictive model $\mathcal{F}_\theta$, we need to estimate the function $f$ mapping from $\boldsymbol{x} \to \boldsymbol{c}$.

In the optimization model $\mathcal{G}$, the matching problem is formulated as:

\begin{mini!}{\boldsymbol{x}}
    {\sum_{i \in I} \sum_{j \in J}   c_{ij} x_{ij} + \epsilon ||\boldsymbol{x}||_2^2} 
    {\label{eq: rbm qp}}{\label{obj: rbm obj}}
    \addConstraint{\sum_{i\in I}{x_{ij}}}{\leq 1}, \forall \ j \in J {\label{cst: sum_i = 1}}
    \addConstraint{\sum_{j\in J}{x_{ij}}}{\leq 1}, \forall \ i\in I {\label{cst: sum_j = 1}}
    \addConstraint{\sum_{i\in I}\sum_{j\in J}{x_{ij}}}{\geq S}, \forall \ i\in I, \forall j \in J {\label{cst: rbm N}}
    \addConstraint{x_{ij}}{\geq 0}, \forall \ i\in I, \forall j \in J {\label{cst: zero}}
    \addConstraint{x_{ij}}{\leq 1}, \forall \ i\in I, \forall j \in J, {\label{cst: one}}
\end{mini!}
where $ \epsilon$ denotes a regularization term, the objective in Equation \ref{obj: rbm obj} minimizes the total cost of picking up riders. Constraints \ref{cst: sum_i = 1} and \ref{cst: sum_j = 1} ensure each driver serves at most one rider, and each request is assigned to at most one driver. The third constraint \ref{cst: rbm N} provides the lower bound ($S$) of in-service riders.

\textbf{Derivation of the Jacobian Matrix of R-BMP}

We follow the same derivation procedures in Appendix A.4 about deriving the Jacobian Matrix inside the implicit differentiation function. For easier implementation, we flatten the two-dimensional decision variable $\boldsymbol{x}, \boldsymbol{c}$ to a one-dimensional variable $\boldsymbol{z}, \boldsymbol{q}$. Then the corresponding $\frac{\partial \boldsymbol{x}^*}{\partial \boldsymbol{c}^*}$ will be replaced by  $\frac{\partial \boldsymbol{z}^*}{\partial \boldsymbol{q}^*}$.

Then the one-dimensional optimization model is given as follows: 
\begin{mini!}{\boldsymbol{z}}
    {\boldsymbol{q}^T \boldsymbol{z} + \epsilon ||\boldsymbol{z}||_2^2} 
    {\label{eq: original qp}}{\label{obj: obj}}
    \addConstraint{\boldsymbol{A}\boldsymbol{z}}{\leq \boldsymbol{1}} {\label{cst: A}}
    \addConstraint{\boldsymbol{B}\boldsymbol{z}}{\leq \boldsymbol{1}} {\label{cst: B}}
    \addConstraint{\boldsymbol{1^T}\boldsymbol{z}}{\geq S} {\label{cst: N}}
    \addConstraint{\boldsymbol{z}}{\geq \boldsymbol{0}} {\label{cst: 0}}
    \addConstraint{\boldsymbol{z}}{\leq \boldsymbol{1},} {\label{cst: 1}}
\end{mini!}
where $\boldsymbol{q} \in \mathbb{R}^{n^2}$, $\boldsymbol{z} = vec(\boldsymbol{x}) \in \mathbb{R}^{n^2}$, $\boldsymbol{A} \in \mathbb{R}^{n \times n^2}$ is the matrix of sum of columns, $\boldsymbol{B} \in \mathbb{R}^{n \times n^2}$ is the matrix of sum of rows, $\boldsymbol{1} \in \mathbb{R}^{n^2}$, \\
$ \boldsymbol{A} = \left[
            \begin{array}{ccc|c|ccc}
                 1 & 0 \cdots &0 &  & 1 & 0 & \cdots 0 \\
                  0& 1 \cdots &0 & & 0 & 1 & \cdots 0 \\
                 \vdots & \vdots & \vdots & \cdots & \vdots & \vdots & \vdots \\ 
                 0& 0 \cdots &1 && 0 & 0 & \cdots 1 \\
            \end{array}
            \right],$\\
$ \boldsymbol{B} = \left[
            \begin{array}{ccc|c|ccc}
                 1 & 1 \cdots &1 &  & 1 & 1 & \cdots 1 \\
                  0& 0 \cdots &0 & & 0 & 0 & \cdots 0 \\
                 \vdots & \vdots & \vdots & \cdots & \vdots & \vdots & \vdots \\ 
                 0& 0 \cdots &0 && 0 & 0 & \cdots 0 \\
            \end{array}
            \right].$

Give the Lagrangian multiplier function $\mathcal{K}$ with dual variables $\boldsymbol{\lambda}, \boldsymbol{\mu}, \gamma, \boldsymbol{\nu}, \boldsymbol{\eta}$ as:
\begin{align}
\small
    \mathcal{K}(\boldsymbol{z},\boldsymbol{\lambda}, \boldsymbol{\mu}, \gamma, \boldsymbol{\nu}, \boldsymbol{\eta};\boldsymbol{c}) = &\boldsymbol{q}^T \boldsymbol{z}  + \epsilon ||\boldsymbol{z}||_2^2 + \boldsymbol{\lambda}^T(A\boldsymbol{z}-\boldsymbol{1}) \nonumber + \nonumber \\
    & \boldsymbol{\mu}^T(B\boldsymbol{z}-\boldsymbol{1}) + \nonumber \gamma(S - \boldsymbol{1}^T\boldsymbol{z}) - \nonumber \\
    & \boldsymbol{\nu}^T \boldsymbol{z} + \boldsymbol{\eta}^T(\boldsymbol{z} - 1),
\end{align}
where $\boldsymbol{\lambda} \in \mathbb{R}^n, \boldsymbol{\mu} \in \mathbb{R}^n, \gamma \in \mathbb{R}, \boldsymbol{\nu} \in \mathbb{R}^{n^2}, \boldsymbol{\eta} \in \mathbb{R}^{n^2}$. 

Then the KKT conditions are presented as follows:
\begin{align}
    &\nabla_{\boldsymbol{z}} \mathcal{K} = \boldsymbol{q}^T + 2\epsilon \boldsymbol{z} + A^T \boldsymbol{\lambda} + B^T\boldsymbol{\mu} - \gamma \boldsymbol{1} - \boldsymbol{\nu} + \boldsymbol{\eta} = 0 \label{eq: stationarity}\\
    &\boldsymbol{\lambda} \geq 0, \boldsymbol{\mu} \geq 0, \gamma \geq 0, \boldsymbol{\nu}\geq 0, \boldsymbol{\eta}\geq 0  \\
    &\boldsymbol{\lambda}_j (A_j\boldsymbol{z} -1) = 0, \quad \forall \ i \in \{1,2,...,n\} \\
    &\boldsymbol{\mu}_i (B_i\boldsymbol{z} -1) = 0, \quad\forall \ i \in \{1,2,...,n\}  \\   
    &\gamma (S - \boldsymbol{1}^T\boldsymbol{z}) = 0  \\ 
    &\boldsymbol{\nu}_i\boldsymbol{z}_i = 0 , \quad\forall \ i \in \{1,2,...,n^2\}  \\
    &\boldsymbol{\eta}_i(\boldsymbol{z}_i -1 )= 0, \quad \ \forall \ i \in \{1,2,...,n^2\}   \label{eq: kkt final}
\end{align}

Construct the KKT condition from Equation \ref{eq: stationarity} to \ref{eq: kkt final} to an implicit function as $\boldsymbol{\psi}(\boldsymbol{z}, \boldsymbol{\lambda}, \boldsymbol{\mu}, \gamma, \boldsymbol{\nu}, \boldsymbol{\eta}; \boldsymbol{q}) = 0$:
\begin{align}
\label{eq: F function}
    &\boldsymbol{\psi}(\boldsymbol{z}, \boldsymbol{\lambda}, \boldsymbol{\mu}, \gamma, \boldsymbol{\nu}, \boldsymbol{\eta}; \boldsymbol{q})= \nonumber \\
    &\left[ 
    \small
        \begin{array}{c}
             \boldsymbol{q}^T + 2\epsilon \boldsymbol{z} + A^T \boldsymbol{\lambda} + B^T\boldsymbol{\mu} - \gamma \boldsymbol{1} - \boldsymbol{\nu} + \boldsymbol{\eta} = 0 \\
             A\boldsymbol{z} - \boldsymbol{1} \leq 0, B\boldsymbol{z} - \boldsymbol{1} \leq 0, N - \boldsymbol{1}^T\boldsymbol{z} \leq 0, \boldsymbol{z} \geq 0, \boldsymbol{z} \leq1  \\ 
             \boldsymbol{\lambda}, \boldsymbol{\mu}, \gamma, \boldsymbol{\nu}, \boldsymbol{\eta} \geq 0 \\ 
             \boldsymbol{\lambda}_j (A_j\boldsymbol{z} -1) = 0, \  \boldsymbol{\mu}_i (B_i\boldsymbol{z} -1) = 0, \ \gamma (S - \boldsymbol{1}^T\boldsymbol{z}) = 0 \\
            \nu_i\boldsymbol{z}_i = 0, \ \eta_i(\boldsymbol{z}_i -1 )= 0, i = 1,2,...,n^2  
        \end{array} 
    \right]
\end{align}

Differentiate the stationarity and the complementary slackness functions w.r.t. $\boldsymbol{q}$:

\begin{align}
\boldsymbol{M}
    \left[ 
        \begin{array}{c}
            d\boldsymbol{z} \\
            d\boldsymbol{\lambda} \\
            d\boldsymbol{\mu}\\
            d\gamma\\
            d\boldsymbol{\nu}  \\
            d\boldsymbol{\eta}\\
        \end{array}
    \right] = 
    \left[ 
        \begin{array}{c}
            -\boldsymbol{I}  \\
            \boldsymbol{0} \\
            \boldsymbol{0} \\
            0 \\
            \boldsymbol{0} \\
            \boldsymbol{0} \\
        \end{array}
    \right] d\boldsymbol{q},     
\end{align}

\begin{align}
\label{M}
&\boldsymbol{M} = \nonumber \\
  & \left[ 
  \small
  \setlength{\arraycolsep}{0.5pt}
        \begin{array}{cccccc}
            2\epsilon \boldsymbol{I} & \boldsymbol{A}^T & \boldsymbol{B}^T & -\boldsymbol{1} & - \boldsymbol{I} & \boldsymbol{I}   \\
            D(\boldsymbol{\lambda})\boldsymbol{A} & D(\boldsymbol{A}\boldsymbol{z} -\boldsymbol{1}) & \boldsymbol{0} & \boldsymbol{0} & \boldsymbol{0} & \boldsymbol{0} \\
            D(\boldsymbol{\mu})\boldsymbol{B} & \boldsymbol{0} & D(\boldsymbol{B}\boldsymbol{z}-\boldsymbol{1}) & \boldsymbol{0} & \boldsymbol{0} & \boldsymbol{0} \\
            (-\gamma \boldsymbol{1}^T) & 0 & 0 & S - \boldsymbol{1}^T\boldsymbol{z} & 0 & 0 \\
            D(\boldsymbol{\nu}) & \boldsymbol{0} & \boldsymbol{0} & \boldsymbol{0} & D(\boldsymbol{z}) & \boldsymbol{0} \\
            D(\boldsymbol{\eta}) & \boldsymbol{0} & \boldsymbol{0} & \boldsymbol{0} & \boldsymbol{0} & D(\boldsymbol{z}-\boldsymbol{1}) \\
        \end{array}      
    \right],
\end{align}
where $D(\cdot)$ denotes the diagonal matrix of $(\cdot)$, and $\boldsymbol{M} \in \mathbb{R}^{(3n^2+2n+1) \times (3n^2+2n+1)}$.

Then we have:
\begin{align}
\label{eq: rh_MI}
    \renewcommand{\arraystretch}{1.5} 
    \left[ 
        \begin{array}{c}
            \frac{\partial \boldsymbol{z}}{\partial \boldsymbol{q}}  \\
            \frac{\partial \boldsymbol{\lambda}}{\partial \boldsymbol{q}} \\
            \frac{\partial \boldsymbol{\mu}}{\partial \boldsymbol{q}} \\
            \frac{\partial \gamma}{\partial \boldsymbol{q}} \\
            \frac{\partial \boldsymbol{\nu}}{\partial \boldsymbol{q}} \\
            \frac{\partial \boldsymbol{\eta}}{\partial \boldsymbol{q}} \\
        \end{array}
    \right] = -\boldsymbol{M}^{-1}
    \left[ 
    \begin{array}{c}
            \boldsymbol{I}  \\
            \boldsymbol{0} \\
            \boldsymbol{0} \\
            0 \\
            \boldsymbol{0} \\
            \boldsymbol{0} \\
    \end{array}    
\right].
\end{align}

Then $\frac{\partial \boldsymbol{z}^*}{\partial\boldsymbol{q}^*}$ can be obtained by the first $n^2$ rows via Equation \ref{eq: rh_MI}. 

\section{Appendix D. Supplementary Results for Experiments}
This section presents the supplementary results of the two numerical experiments.

\subsection{D.1 Implementation Details}

In the numerical experiments of the two datasets, we implement the predictive model in the R-DFL framework as an MLP with a hidden dimension of 32, using LeakyReLU activation and a dropout rate of 0.1 for regularization. All weights and biases are initialized using the Kaiming uniform method. The learning rate is $1e-3$. All decision variables are continuous. The hyperparameters are set as: the batch size is 8, the optimizer is Adam with weight decay $5e-4$. The dataset is split by $8:1:1$ for training, validation, and testing. The R-DFL framework with explicit unrolling methods is configured with 10 unrolling layers and a convergence tolerance of 0.2 for the implicit differentiation iterations. All the experiments are conducted on a desktop with Intel Core i7-13700K CPU 3.40 GHz $\times$ 32G RAM, 500 GB SSD, GeForce RTX 3090 Ti GPU.

\subsection{D.2 Training Curves}
This section presents the training curves of the R-DFL framework with two differentiation methods across small-scale newsvendor and bipartite matching problems. One can see that both explicit and implicit differentiation methods of R-DFL converge after 50 epochs. 

\begin{figure}[ht]
    \centering
    \includegraphics[width=0.95\linewidth]{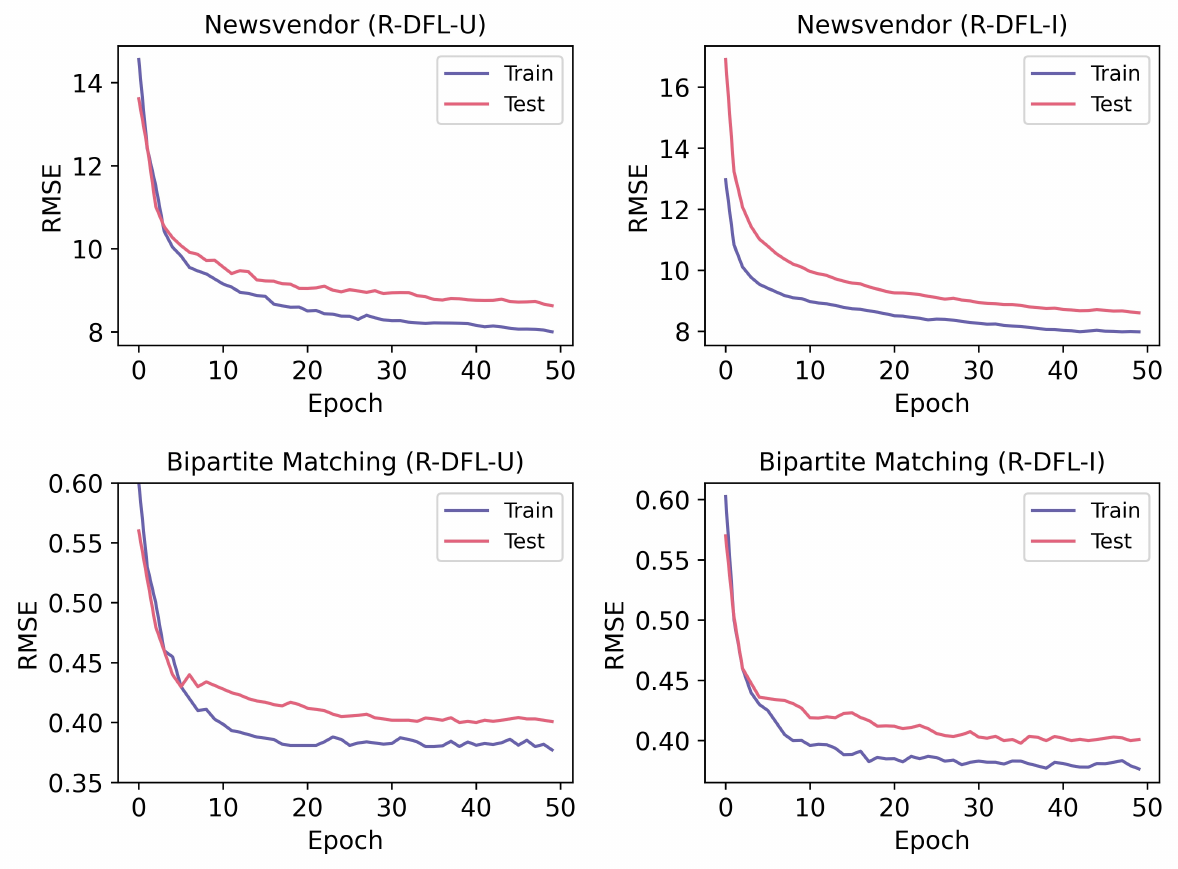}
    \caption{Training curves of R-DFL-U and R-DFL-I.}
    \label{fig:placeholder}
\end{figure}

\subsection{D.3 Performance of Different Predictive Models}
Table \ref{tab: full table} (See page 17) presents the complete accuracy and efficiency comparison of different predictive models of the R-DFL framework and other baselines using a 10-layer unrolling structure across the two datasets. The predictive models are: \texttt{LSTM} \citep{hochreiter1997long}, \texttt{RNN} \citep{elman1990finding}, and \texttt{Transformer} \citep{vaswani2017attention}. The results demonstrate that the R-DFL framework achieves higher accuracy than frameworks that neglect bidirectional feedback (S-DFL and PTO), highlighting the importance of bidirectional structures in closed-loop decision-making. Additionally, R-DFL-I exhibits significantly better computational efficiency than R-DFL-U across nearly all experiments, with relatively lower bias in standard deviation, highlighting the robustness of the proposed implicit differentiation methods.

\clearpage
\begin{landscape}
\begin{table}[h]
    \small 
    \setlength{\tabcolsep}{2 pt} 
    \begin{tabular}{l|cc|cc|cc|cc|cc|cc}
        \toprule
        Dataset & \multicolumn{6}{c|}{\textbf{Newsvendor Problem}} & \multicolumn{6}{c}{\textbf{Bipartite Matching Problem}} \\
        \midrule
        Scalability & \multicolumn{2}{c|}{Small}& \multicolumn{2}{c|}{Mid} & \multicolumn{2}{c|}{Large} & \multicolumn{2}{c|}{Small} & \multicolumn{2}{c|}{Mid} & \multicolumn{2}{c}{Large} \\
        \midrule
        Decision variable& \multicolumn{2}{c|}{10} & \multicolumn{2}{c|}{50}& \multicolumn{2}{c|}{100} & \multicolumn{2}{c|}{16} & \multicolumn{2}{c|}{225}& \multicolumn{2}{c}{900} \\
        Constraints & \multicolumn{2}{c|}{32} & \multicolumn{2}{c|}{152} & \multicolumn{2}{c|}{302} & \multicolumn{2}{c|}{57} & \multicolumn{2}{c|}{706} & \multicolumn{2}{c}{2761} \\
        Jacobian matrix & \multicolumn{2}{c|}{$32\times 32$} & \multicolumn{2}{c|}{$152 \times 152$} & \multicolumn{2}{c|}{$302 \times 302$} & \multicolumn{2}{c|}{$57 \times 57$}& \multicolumn{2}{c|}{$706 \times 706$} & \multicolumn{2}{c}{2761 $\times$ 2761} \\
        \midrule
        \multicolumn{13}{c}{MLP} \\
        \midrule
        Metrics & RMSE & Time & RMSE& Time & RMSE& Time & RMSE & Time & RMSE& Time & RMSE& Time \\
        \midrule
        PTO & 12.771 $\pm$ 0.731   & - & 12.747 $\pm$ 0.808 & -  & 12.684 $\pm$ 0.876 & -   & 0.412 $\pm$ 0.012 & -   & 0.232 $\pm$ 0.016 & -   & 0.190 & -  \\
        S-DFL & 12.245  $\pm$ 0.694 & -  & 12.536  $\pm$ 0.716 & -  & 12.649  $\pm$ 0.912 & -   & 0.408 $\pm$ 0.010 & -   & 0.231 $\pm$ 0.019 & -   & 0.187  $\pm$ 0.015 & -   \\
       \underline{\textbf{R-DFL-U}} & 8.983  $\pm$ 0.494 & 135 $\pm$ 12 & 9.173  $\pm$ 0.302 & 369  $\pm$ 23 & 9.343  $\pm$ 0.470 & 422  $\pm$ 29 & \underline{\textbf{0.396  $\pm$ 0.028}} & 65  $\pm$ 4 & 0.222  $\pm$ 0.020 & 432  $\pm$ 13 & \underline{\textbf{0.170  $\pm$ 0.012 }} &  2704  $\pm$ 10 \\
        \underline{\textbf{R-DFL-I}} & \underline{\textbf{8.831  $\pm$ 0.461}} & \underline{\textbf{118 $\pm$ 10}} & \underline{\textbf{9.106 $\pm$ 0.565 }}  & \underline{\textbf{254 $\pm$ 9}} & \underline{\textbf{9.327 $\pm$ 0.851}}  & \underline{\textbf{369  $\pm$ 22 }} & 0.398  $\pm$ 0.022 & \underline{\textbf{26  $\pm$ 3}}  & \underline{\textbf{0.220  $\pm$ 0.028}} & \underline{\textbf{65} $\pm$ 4} & 0.171  $\pm$ 0.013  & \underline{\textbf{1867 $\pm$ 34}}  \\
        \midrule
        \multicolumn{13}{c}{LSTM } \\
        \midrule
        Metrics & RMSE & Time & RMSE& Time & RMSE& Time & RMSE & Time & RMSE& Time & RMSE& Time \\
        \midrule
        PTO & 13.034 $\pm$ 0.391   & - & 12.964 $\pm$ 0.682 & - & 12.034 $\pm$ 0.535 & -& 0.411 $\pm$ 0.024 & - & 0.218 $\pm$ 0.019 & - & 0.219  $\pm$ 0.013 & - \\
        S-DFL & 13.070  $\pm$ 0.363 & -  & 12.471  $\pm$ 0.727  & - & 12.693 $\pm$ 0.471 & -& 0.421$\pm$ 0.012 & - &  0.265 $\pm$ 0.014 & -  & 0.230  $\pm$ 0.012 & - \\
       \underline{\textbf{R-DFL-U}} & 8.544  $\pm$ 0.407 & 426 $\pm$ 21 & 10.357  $\pm$ 0.524 & 585  $\pm$ 43 & 10.112  $\pm$ 0.753 & 531 $\pm$ 29 & 0.401 $\pm$ 0.039 & 80 $\pm$ 7 & \underline{\textbf{0.172 $\pm$ 0.017}} & 274 $\pm$ 29 & 0.176 $\pm$ 0.011 & 2704 $\pm$ 102 \\
        \underline{\textbf{R-DFL-I}} & \underline{\textbf{8.485  $\pm$ 0.393}} & \underline{\textbf{370 $\pm$ 18}} & \underline{\textbf{10.198  $\pm$ 0.672}} & \underline{\textbf{484  $\pm$ 38}} & \underline{\textbf{10.104 $\pm$ 0.373}} & \underline{\textbf{355 $\pm$ 19}} & \underline{\textbf{0.389 $\pm$ 0.026}}& \underline{\textbf{23 $\pm$ 3}} & 0.174 $\pm$ 0.014 & \underline{\textbf{73 $\pm$ 7}} & \underline{\textbf{0.174 $\pm$ 0.011}} & \underline{\textbf{2093  $\pm$ 87}} \\
        \midrule
        \multicolumn{13}{c}{RNN} \\
        \midrule
        Metrics & RMSE & Time & RMSE& Time & RMSE& Time & RMSE & Time & RMSE& Time & RMSE& Time \\
        \midrule
        PTO & 13.015 $\pm$ 0.698 & - & 12.748 $\pm$ 0.632 & - & 12.842 $\pm$ 1.110 & - & 0.405 $\pm$ 0.021 & - & 0.224 $\pm$ 0.015 & - & 0.185 $\pm$ 0.011 & - \\
        S-DFL & 12.819 $\pm$ 0.307 & -  & 12.961 $\pm$ 0.420 & - & 12.583 $\pm$ 0.939 & - & 0.411 $\pm$ 0.016 & - & 0.224 $\pm$ 0.013 & - & 0.187 $\pm$ 0.027 & - \\
       \underline{\textbf{R-DFL-U}} & 8.950 $\pm$ 0.462 & 621 $\pm$ 34 & 10.987 $\pm$ 0.834 & 680  $\pm$ 36 & 10.872 $\pm$ 0.390 & 510  $\pm$ 46 & \underline{\textbf{0.397 $\pm$ 0.011}} & 102 $\pm$ 11 & 0.176 $\pm$ 0.019 & 300 $\pm$ 27 & 0.176 $\pm$ 0.012 & 2821 $\pm$ 104  \\
        \underline{\textbf{R-DFL-I}} & \underline{\textbf{8.670$\pm$ 0.431}} &\underline{\textbf{431 $\pm$ 25}} &\underline{\textbf{10.976$\pm$ 0.687}} & \underline{\textbf{555 $\pm$ 41}}& \underline{\textbf{10.810$\pm$ 0.692}} & \underline{\textbf{340 $\pm$ 41}} & 0.399 $\pm$ 0.013 & \underline{\textbf{31 $\pm$ 4}} & \underline{\textbf{0.174 $\pm$ 0.019 }} & \underline{\textbf{91$\pm$12 }} & \underline{\textbf{0.174 $\pm$ 0.017 }} & \underline{\textbf{2079 $\pm$ 99 }}\\
        \midrule
        \multicolumn{13}{c}{Transformer} \\
        \midrule
        Metrics & RMSE & Time & RMSE& Time & RMSE& Time & RMSE & Time & RMSE& Time & RMSE& Time \\
        \midrule
        PTO & 12.577 $\pm$ 0.892  & - & 13.694 $\pm$ 1.103 & - & 14.071 $\pm$ 0.998 & - & 0.412 $\pm$ 0.015 & - & 0.224 $\pm$ 0.017 & - & 0.193 $\pm$ 0.017 & - \\
        S-DFL & 12.183 $\pm$ 0.938 & -  & 13.412 $\pm$ 1.006 & - & 14.040 $\pm$ 1.104 & - & 0.413 $\pm$ 0.027 & - & 0.225 $\pm$ 0.014 & - & 0.188 $\pm$ 0.023 & - \\
       \underline{\textbf{R-DFL-U}} & 10.167 $\pm$ 0.841 & 562 $\pm$ 37 & 11.300  $\pm$ 0.719 & 541 $\pm$ 39 & \underline{\textbf{11.231 $\pm$ 0.867}} & 561 $\pm$ 41 & \underline{\textbf{0.405 $\pm$ 0.028}}& 130 $\pm$ 9 & \underline{\textbf{0.179$\pm$ 0.011}} & 293 $\pm$ 15 & 0.172 $\pm$ 0.010 & 2821 $\pm$ 154 \\
        \underline{\textbf{R-DFL-I}} & \underline{\textbf{10.167 $\pm$ 0.754}} & \underline{\textbf{532 $\pm$ 33}}& \underline{\textbf{11.142 $\pm$ 0.710}}& \underline{\textbf{489 $\pm$ 32}} & \underline{\textbf{11.332$\pm$ 0.838}} & \underline{\textbf{360$\pm$25}}& \underline{\textbf{0.407$\pm$ 0.014}}& \underline{\textbf{37$\pm$ 4}} & 0.180 $\pm$ 0.013 & \underline{\textbf{78 $\pm$ 9}} & \underline{\textbf{0.166 $\pm$ 0.008}} & \underline{\textbf{2078 $\pm$ 167}}\\
        \bottomrule
    \end{tabular}
    \caption{Performance of R-DFL framework with explicit unrolling and implicit differentiation methods on the newsvendor and bipartite matching problem with different predictive models. Jacobian size indicates dimension of matrix $J_{\mathcal{G}}|_{\boldsymbol{c}_{i+1}}$). Unit for time: second.}
\label{tab: full table}
\end{table}
\end{landscape}

\end{document}